\documentclass{article}

\PassOptionsToPackage{numbers,compress}{natbib}

\usepackage[preprint]{neurips_2026}

\usepackage[utf8]{inputenc}
\usepackage[T1]{fontenc}
\usepackage{amsmath,amssymb,amsfonts,amsthm,bm}
\usepackage{booktabs}
\usepackage{graphicx}
\usepackage{subcaption}
\usepackage{titletoc}
\usepackage{hyperref}
\usepackage{url}
\usepackage{microtype}
\usepackage{xcolor}
\usepackage[breakable]{tcolorbox}
\usepackage{listings}

\title{Exactness Matters for Physical Rule Enforcement}
\author{%
  Bum Jun Kim\thanks{Corresponding author.}\\
  The University of Tokyo, Japan\\
  \texttt{bumjun.kim@weblab.t.u-tokyo.ac.jp}
}
\date{}

\newcommand{\R}{\mathbb{R}}
\newcommand{\uu}{\bm{u}}
\newcommand{\yy}{\bm{y}}
\newcommand{\kk}{\bm{k}}
\newcommand{\divergence}{\nabla \cdot}
\newcommand{\proj}{\Pi_{\mathrm{div}}}
\newcommand{\aproxproj}{\widetilde{\Pi}_{\mathrm{bc}}}
\newcommand{\clip}{\operatorname{clip}}
\newcommand{\mseDivCell}[2]{\shortstack[c]{$#1$\\$#2$}}
\newtheoremstyle{boxedplain}
  {0pt}{0pt}{\itshape}{}{\bfseries}{.}{.5em}{}
\newtheoremstyle{boxeddefinition}
  {0pt}{0pt}{\normalfont}{}{\bfseries}{.}{.5em}{}
\newtheoremstyle{boxedremark}
  {0pt}{0pt}{\normalfont}{}{\itshape}{.}{.5em}{}
\theoremstyle{boxedplain}
\newtheorem{theorem}{Theorem}
\newtheorem{proposition}[theorem]{Proposition}

\newtheorem{corollary}[theorem]{Corollary}
\theoremstyle{boxeddefinition}

\theoremstyle{boxedremark}
\newtheorem{remark}{Remark}
\tcbset{
  formalgraybox/.style={
    breakable,
    colback=gray!10,
    colframe=gray!10,
    boxrule=0pt,
    arc=0pt,
    outer arc=0pt,
    boxsep=0pt,
    left=0pt,
    right=0pt,
    top=0pt,
    bottom=0pt,
    before skip=.7\baselineskip,
    after skip=.7\baselineskip,
  }
}
\tcolorboxenvironment{theorem}{formalgraybox}
\tcolorboxenvironment{proposition}{formalgraybox}
\tcolorboxenvironment{lemma}{formalgraybox}
\tcolorboxenvironment{corollary}{formalgraybox}
\tcolorboxenvironment{definition}{formalgraybox}
\tcolorboxenvironment{example}{formalgraybox}
\tcolorboxenvironment{remark}{formalgraybox}
\lstdefinestyle{pythoncode}{
  language=Python,
  basicstyle=\ttfamily\scriptsize,
  keywordstyle=\color{blue!60!black},
  commentstyle=\color{gray!70!black},
  stringstyle=\color{green!40!black},
  showstringspaces=false,
  breaklines=true,
  columns=fullflexible,
  keepspaces=true,
  frame=single,
  framerule=0.2pt,
  rulecolor=\color{gray!35},
  xleftmargin=1em,
  xrightmargin=1em,
  aboveskip=0.75\baselineskip,
  belowskip=0.75\baselineskip,
  captionpos=b
}

\begin{document}
\maketitle

\begin{abstract}
Autoregressive scientific forecasters often enforce physical or structural constraints by repairing each predicted state before feeding it back into the model. However, it remains unclear when stronger physical rule enforcement becomes reliable and when it becomes a source of distribution shift. We study this question through operator exactness, meaning whether the repair map is the identity on the target manifold and is aligned with the target geometry. We compare raw forecasting, post hoc repair, and in-loop repair across periodic incompressible Navier--Stokes, non-periodic CFDBench flows, and a hierarchical-forecasting support task. In the exact periodic regime, Fourier projection substantially improves rollout accuracy. On the NS-128 benchmark, a strong Raw-FNO has a final-step rollout MSE at horizon 100 of $(9.390 \pm 6.290)\times 10^{-5}$, and post hoc and in-loop projection reduce it to $(1.130 \pm 0.165)\times 10^{-6}$ and $(5.370 \pm 0.113)\times 10^{-7}$. However, once an exact projection is unavailable and only approximate boundary-preserving cleanup is available, the ordering changes. Across cavity, tube, dam, and cylinder flow, stronger Poisson-based cleanup can reduce divergence while worsening rollout error; target-distortion MSE predicts this harm far better than a linear-system residual. Controlled mismatch, screened cleanup, adaptive gating, and external-backbone checks show that the best approximate-regime operating point can be raw or near-identity. Hierarchical forecasting gives the same broader pattern. Exact forecast reconciliation is a stable baseline, whereas blended top-down repair, a validation-tuned interpolation toward historical-proportion top-down reconciliation, is dataset-dependent. Thus, constraint enforcement should be benchmarked by operator--data alignment before enforcement strength. Use in-loop projection when the operator is exact, and validate approximate cleanup strength using rollout metrics otherwise.
\end{abstract}

\section{Introduction}
Autoregressive scientific forecasting is often stabilized by a simple repair loop using a physical rule. A model first predicts the next state; a projection or cleanup operator then repairs that prediction so that it satisfies a known physical or structural rule; the repaired state is finally fed back into the next prediction step. In incompressible flow, as an example, the rule is the divergence-free constraint $\divergence \uu = 0$, and the repair loop is especially tempting because small compressible components can accumulate during rollout, corrupt transport, and distort long-horizon energy and vorticity statistics. This stabilization benefit is why learned-simulator papers increasingly report long-horizon stability and physical consistency in addition to one-step error \citep{kochkov2021machine,sanchezgonzalez2020learning,brandstetter2022message,stachenfeld2022learned,lippe2023pderefiner,koehler2024apebench}. Nevertheless, as we show in Section~\ref{sec:experiments}, overly aggressive enforcement can become a recurring source of distribution shift when the repair operator is only approximate.

We study when this repair loop is a reliable improvement. Our conclusion is that reliability depends on whether the repair operator leaves valid benchmark targets unchanged. If the operator is exact on the target manifold, then it leaves valid targets unchanged and removes only components that the benchmark regards as irrelevant. In this regime, projection is a strong and often mandatory baseline. For periodic Navier--Stokes, Fourier projection maps predictions back into the divergence-free subspace without moving divergence-free targets.

However, this practice can cause failure when the available constraint operator is only approximate. This approximate regime is not a corner case. Soft residual penalties, learned pressure or projection modules, and boundary-dependent fractional-step or projection schemes have all been used to impose incompressibility or related physical structures in practice \citep{raissi2019physics,tompson2017accelerating,li2021physics,jin2021nsfnets,brown2001accurate,guermond2006overview}. Non-periodic benchmarks with copied boundaries, complex domains, or discrete pressure cleanups may not admit the exact periodic projector used above. The key risk is that an alternative, such as a copied-boundary Poisson solve, may improve the visible constraint residual without preserving the data manifold. It can reduce measured divergence while still perturbing physically valid target states or shifting the autoregressive state distribution seen by the model. Solving this approximate cleanup more aggressively then makes the chosen operator more faithful, but it does not make the operator better aligned with the data. Stronger enforcement can therefore worsen rollout error even as the constraint residual improves. Figure~\ref{fig:operator-exactness} summarizes this exact-versus-approximate split, which is the main subject of our study. This split is not merely conceptual. It appears empirically in rollout evaluations and is supported by the corresponding exactness and distortion analyses in Section~\ref{sec:experiments} and Appendix~\ref{app:theory}.

\begin{figure}[t!]
\centering
\includegraphics[width=\textwidth]{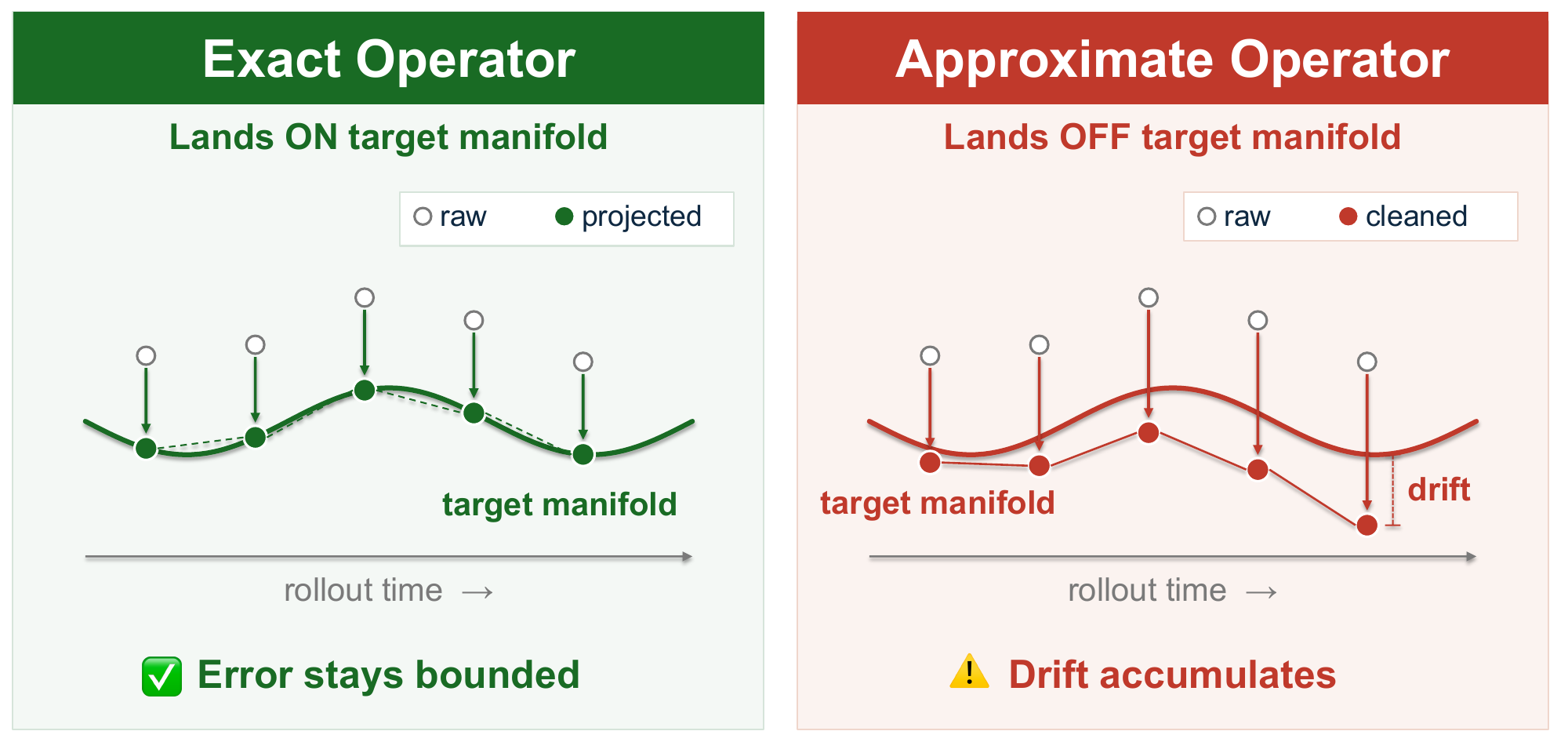}
\caption{Conceptual split between exact and approximate constraint operators in autoregressive rollout. The left side shows the exact case: projected states remain on the target manifold, allowing constraint correction to keep rollout error bounded. The right side shows the approximate case: the cleaned state can land off the target manifold, so repeated correction may accumulate drift even while it reduces the measured constraint violation.}
\label{fig:operator-exactness}
\end{figure}

On strong periodic velocity benchmarks, most of the gain comes not from a bespoke decoder but from applying this exact projection to each predicted state before feeding it into the next step. This exact-projection result changes the benchmarking standard. Any new constrained architecture must be compared not only to raw forecasting but also to post hoc and in-loop projection. Our main claim therefore concerns baseline quality under exact and approximate constraint operators, with architecture comparisons used to test whether the regime split persists across different backbones.

We then test, through a focused study using the official splits for the non-periodic cavity, tube, dam, and cylinder subsets of the CFDBench CFD benchmark introduced by \citet{luo2023cfdbench}, what survives once the exact periodic geometry is removed. Our CFDBench study includes stronger-solver negative controls, screened and geometry-aware cleanup families, long-horizon reevaluations of dam and cylinder flow, controlled mismatch sweeps, adaptive test-time cleanup, and a frozen U-Net transfer check. Appendix~\ref{app:hierarchical-support} extends the evidence with hierarchical-forecasting experiments on five public datasets. Together, these results define an empirical regime analysis of exact and approximate constraint operators, centered on the baseline rule rather than on a new forecasting architecture.

\paragraph{Contributions.}
We identify a regime split in constraint enforcement under strong autoregressive baselines. For periodic Navier--Stokes (NS-128), exact in-loop projection is the dominant intervention, and post hoc projection is a mandatory comparator; on non-periodic CFDBench, this ordering breaks once cleanup is only approximate. Mechanism diagnostics, including stronger-solver controls, target-distortion measurements, screened sweeps, and controlled mismatch experiments, show that approximate-regime failures are driven by operator mismatch rather than by the Poisson residual. Appendix~\ref{app:hierarchical-support} gives a compact non-CFD check in which exact reconciliation is the lowest-variance baseline across five public hierarchies, while blended top-down repair, a validation-tuned partial interpolation toward historical-proportion top-down reconciliation, can help or hurt. These results yield a practical guideline. Use in-loop projection when the operator is exact and orthogonal; otherwise, benchmark exact-on-target operators when available, and validate approximate cleanup strength using held-out rollout metrics. Our evidence is strongest for two-dimensional (2D) incompressible forecasting, with the hierarchical study serving as supporting structured-prediction evidence.

\section{Related Work}
\paragraph{Operator learning and neural simulators.}
The Fourier neural operator (FNO) of \citet{li2021fourier} remains a standard backbone for periodic partial differential equation (PDE) forecasting, while the DeepONet line of \citet{lu2021deeponet} provides a second canonical operator-learning family. U-Net-style encoder-decoder models remain competitive in scientific surrogate modeling because of their strong locality bias, and convolutional-network fluid surrogates remain important baselines in learned fluid simulation \citep{ronneberger2015u,kim2019deepfluids}. Learned simulators based on message passing and graph networks emphasize geometry and conservation structure \citep{sanchezgonzalez2020learning,pfaff2020learning,brandstetter2022message}, while recent rollout-focused PDE work makes long-horizon stability itself a first-class evaluation target \citep{lippe2023pderefiner}. Our goal here is not to propose yet another backbone family but to identify the strongest constraint-centric intervention that all of them should be benchmarked against.
Hybrid neural and solver surrogates and physics-constrained autoregressive PDE models also increasingly couple learned predictors to explicit numerical structure rather than relying on fitting raw data alone \citep{tompson2017accelerating,schenck2018spnets,deavilabelbuteperes2020combining,geneva2020physicsconstrained}.

\paragraph{Physics constraints for incompressible dynamics.}
Soft penalties, residual losses, stream-function parameterizations, projection layers, and explicit test-time constraint solvers are all common ways to encode incompressibility or related physical structure \citep{raissi2019physics,tompson2017accelerating,li2021physics,rubanova2022constraint,jin2021nsfnets}. The problem is less the existence of these techniques than the standard for evaluating them. If post hoc projection is omitted or if a new constrained decoder is only compared to raw forecasting, then empirical conclusions are difficult to interpret. Our adaptive-cleanup and risk--coverage analyses also connect this literature to a broader structured-prediction issue: constraint enforcement must be judged by both feasibility and test-time decision making \citep{geifman2017selective}.

\paragraph{Constraint repair beyond CFD.}
Forecast reconciliation in hierarchical time series provides a clean non-CFD instance of exact and approximate repair. Coherent forecasts lie in the column space of a known summing matrix, yielding exact reconciliation operators such as ordinary least squares (OLS) and minimum-trace reconciliation together with exact-but-nonorthogonal bottom-up aggregation, while top-down operators are generally mismatched to the target manifold \citep{hyndman2011optimal,wickramasuriya2019optimal,panagiotelis2021geometric,athanasopoulos2024review}. Classical top-down and bottom-up heuristics remain important comparators in that literature \citep{gross1990disaggregation,athanasopoulos2009hierarchical}. Recent probabilistic, quantile, temporal, and softly coherent variants extend the same coherence issue into modern machine-learning forecasting pipelines \citep{taieb2017coherent,rangapuram2021end,han2021simultaneously,rangapuram2023temporal,kamarthi2023rigidity,tsiourvas2024learning}. This literature is useful here because it isolates constraint-operator exactness outside fluid dynamics.

\paragraph{Benchmarks beyond periodic NS.}
Most projection-centric empirical claims are easiest to make on periodic 2D velocity benchmarks. CFDBench, introduced by \citet{luo2023cfdbench}, provides a stronger complementary stress test with varied boundary conditions, physical properties, and domain geometries across classic CFD problems. Recent benchmark work on autoregressive PDE emulators likewise emphasizes rollout-aware evaluation rather than one-step accuracy alone \citep{koehler2024apebench}. We use CFDBench's cavity-flow, tube-flow, dam-break, and cylinder-flow subsets to test what remains true once exact Fourier-space projection is no longer available.

\section{Method}
\subsection{Forecasting Setup}
Given an input history
\begin{equation}
\uu_{t-T_{\mathrm{in}}+1:t} \in \R^{T_{\mathrm{in}}\times 2 \times H \times W},
\end{equation}
a backbone $b_\theta$ predicts residual updates
\begin{equation}
\Delta \tilde{\uu}_{t+1:t+K} = b_\theta(\uu_{t-T_{\mathrm{in}}+1:t}),
\end{equation}
and the state evolves through
\begin{align}
\tilde{\uu}_{t+s} &= \hat{\uu}_{t+s-1} + \alpha \Delta \tilde{\uu}_{t+s},\\
\hat{\uu}_{t+s} &= \mathcal{T}(\tilde{\uu}_{t+s}), \text{ for } s=1,\dots,K,
\end{align}
where $\alpha \in (0,1)$ is a learned residual scale and $\mathcal{T}$ is either the identity, an exact divergence-free projector, or an approximate boundary-aware cleanup operator.

\paragraph{Terminology.}
Throughout the paper, exactness is an operator--data property, not a statement about how accurately an internal numerical system is solved. A repair operator $T$ is exact for a benchmark target set $\mathcal{M}$ when $T(y)=y$ for every valid target $y\in\mathcal{M}$; approximate cleanup means that this identity can fail on benchmark-valid targets, even if the cleanup maps many inputs closer to the constraint set. Under this terminology, the Direct CFDBench variant is an exact solve of its boundary-aware Poisson system, but it remains an approximate cleanup operator when that boundary-aware system is not exact-on-target for the benchmark velocity fields.

\subsection{Exact Periodic Projection}
For periodic 2D velocity fields $\uu=(u,v)$, we implement the exact projector in Fourier space. The Fourier-space implementation is the periodic spectral version of the classical projection viewpoint for incompressible flow \citep{chorin1968numerical,temam1969approximation}. Let $\bar{\uu}$ denote the spatial mean of $\uu$, and write $\kk=(k_x,k_y)$ for the two-dimensional Fourier wavevector. Define the Fourier-space vorticity by
\begin{equation}
\hat{\omega}(\kk) = i k_x \hat{v}(\kk) - i k_y \hat{u}(\kk),
\end{equation}
and define the stream function by
\begin{equation}
\hat{\psi}(\kk) = \frac{\hat{\omega}(\kk)}{\|\kk\|_2^2}\quad\text{for }\kk \neq 0,\qquad \hat{\psi}(0)=0.
\end{equation}
The projected velocity is then
\begin{equation}
\proj(\uu) = \bar{\uu} + (\partial_y \psi, -\partial_x \psi),
\end{equation}
which preserves the zero Fourier mode and satisfies $\divergence \proj(\uu)=0$ up to numerical precision.

This exact periodic projector is orthogonal with respect to the $L_2$ geometry of the periodic benchmark. If $\yy=\proj(\yy)$, then, for any raw prediction $\tilde{\uu}$,
\begin{equation}
\|\proj(\tilde{\uu})-\yy\|_2^2
= \|\tilde{\uu}-\yy\|_2^2 - \|(I-\proj)\tilde{\uu}\|_2^2
\le \|\tilde{\uu}-\yy\|_2^2.
\label{eq:orthogonal_proj}
\end{equation}
Equation~\ref{eq:orthogonal_proj} explains why post hoc projection is already a strong baseline under the exact periodic projector. Any difference between projecting both the target and the autoregressive state and projecting only the autoregressive state is therefore a separate training-dynamics issue rather than a consequence of the orthogonal-projection identity itself, and in our exact periodic experiments, it is empirically only a second-order effect.
In Appendix~\ref{app:theory}, Proposition~\ref{prop:orthogonal_projection_identity} proves the general version of Equation~\ref{eq:orthogonal_proj}. Projection onto a subspace containing the target cannot increase the squared error.

\subsection{Boundary-Preserving Approximate Cleanup}
However, for non-periodic cavity, tube, dam, and cylinder flow, exact spectral projection is unavailable. We therefore use training-matched boundary-preserving approximate cleanup operators, denoted $\aproxproj$, which are boundary-constrained Chorin--Temam-style cleanup maps rather than exact Leray projectors. These maps copy boundary velocities and update only interior velocities using a Poisson or screened Poisson pressure solve, so copying raw boundary values can break both exact divergence annihilation and idempotence on the full velocity field. We study a boundary-aware Poisson--Hodge family, including Jacobi, MG, SOR, CG, and Direct variants, and a screened family that weakens the pressure response. Appendix~\ref{app:boundary-cleanup-details} gives the full discrete stencil, solver definitions, and screened-family protocol.

\paragraph{Approximate-regime bias decomposition.}
Unlike $\proj$, the boundary-preserving approximate cleanup operator $\aproxproj$ is not orthogonal. Equation~\ref{eq:orthogonal_proj} does not apply, and neither post hoc nor in-loop cleanup is guaranteed to improve one-step MSE. A useful decomposition for any approximate cleanup operator $\widetilde{\Pi}$ and target $y$ is
\begin{equation}
\|\widetilde{\Pi}(x)-y\|_2^2
=
\|\widetilde{\Pi}(x)-\widetilde{\Pi}(y)\|_2^2
+ \|\widetilde{\Pi}(y)-y\|_2^2
+ 2\langle \widetilde{\Pi}(x)-\widetilde{\Pi}(y), \widetilde{\Pi}(y)-y\rangle.
\label{eq:approx_decomp}
\end{equation}
The second term is the target-distortion term, while the third cross-term can have either sign. Writing the in-loop dynamics as $x_{t+1}=\widetilde{\Pi}(F(x_t))$, with a target trajectory $(y_t)$ and error $e_t=x_t-y_t$, one generic perturbation bound is
\begin{equation}
\|e_{t+1}\|
\le
L\|e_t\|
+ \|\widetilde{\Pi}(F(y_t)) - F(y_t)\|
+ \delta_t,
\label{eq:approx_rollout_bias}
\end{equation}
where $L$ is a local Lipschitz factor of $\widetilde{\Pi}\circ F$ and $\delta_t=\|F(y_t)-y_{t+1}\|$ is the model's intrinsic one-step prediction error on the target trajectory. If $F(y_t)$ stays near the target manifold, the middle term is controlled by the local operator bias near that manifold. Equation~\ref{eq:approx_rollout_bias} is deliberately simple, but it captures our empirical observation. Lowering the linear-system residual only means that the chosen operator is solved more faithfully; it does not guarantee that the operator itself is aligned with the benchmark manifold. This operator-alignment gap is exactly why the non-periodic CFDBench benchmarks are scientifically useful. They test what remains true once the clean periodic geometry is removed. For the non-periodic experiments, we apply all cleanup operators in physical velocity units, and only then map them back into the normalized model space. Conceptually, physical-space cleanup is the cleaner choice because CFDBench tube, dam, and cylinder have nonzero mean flow, so a cleanup operator defined in normalized coordinates changes both the copied boundary values and the effective correction magnitude. Empirically, however, a dedicated negative-control ablation shows that under our canonical velocity-component normalization, the frame choice changes final-horizon rollout MSE by at most $0.14\%$ relative. We therefore use physical-space cleanup as protocol hygiene rather than as a hidden source of gains.

\paragraph{Theory appendix.}
Appendix~\ref{app:theory} formalizes the mechanism separating exact and approximate operators used throughout the experiments, including exact-on-target repair, solver residual, trajectorywise target distortion, and controlled mismatch.

\subsection{Forecasting Variants}
We compare raw forecasting, post hoc repair, in-loop repair, constraint-aware projection, soft-divergence training, and adaptive blended cleanup as experimental variants. Appendix~\ref{app:forecasting-variant-details} gives the full variant definitions, objective, operator list, and matched-budget protocol.

\subsection{Divergence-Gated Adaptive Cleanup}
Approximate projection helps because it removes compressible components, but it hurts because it distorts states away from the training manifold. This trade-off suggests that the right action outside the exact regime is often neither to skip cleanup entirely nor to apply it at full strength. We therefore evaluate an adaptive blended cleanup that uses predicted root-mean-square (RMS) divergence as a risk score:
\begin{equation}
d(\tilde{\uu}) = \|\divergence_h \tilde{\uu}\|_{\mathrm{rms}},\quad \gamma(\tilde{\uu}) = \min\{(d(\tilde{\uu})/\tau)^q, 1\},\quad \hat{\uu}_{\mathrm{ada}} = \tilde{\uu} + \gamma(\tilde{\uu})(\mathcal{P}(\tilde{\uu}) - \tilde{\uu}),
\end{equation}
where $\tau>0$ and $q>0$. When $d(\tilde{\uu})$ is small, the model stays close to the raw forecast; when it is large, the update approaches full post hoc cleanup. We instantiate $\mathcal{P}$ with the Jacobi family and with the screened family for the main deployment study, use Direct only as a stronger-operator stress test, and report additional GeoScreened deployment variants in Appendix~\ref{app:approx-extensions}. We tune the threshold $\tau$ and exponent $q$ once per dataset on the validation split, sweep those parameters to assess sensitivity, and then evaluate the selected rule on the held-out test trajectories. The method requires no retraining and directly operationalizes the safety motivation for using divergence as a warning signal.

\section{Experiments}
\label{sec:experiments}
\paragraph{Datasets.}
We use three main experiment groups in the main text.
The exact periodic main track uses periodic NS-128 with a matched FNO family trained for 18 epochs and evaluated out to 200 rollout steps. The approximate official CFDBench tracks cover cavity, tube, dam, and cylinder under the official splits with matched raw and projected FNO recipes and 20-step rollout evaluation, and dam together with cylinder additionally include 60-step long-horizon reevaluations. Additional approximate-regime experiments add stronger-solver negative controls, adaptive cleanup, controlled mismatch sweeps, the validation-to-deployment protocol below, and the appendix-only geometry-aware, projection-frame, and external-backbone studies collected in Appendix~\ref{app:approx-extensions}. Appendix~\ref{app:hierarchical-support} adds a real-data hierarchical forecasting support track using public Traffic, Labour, TourismLarge, OldTourismLarge, and Wiki2 hierarchies from the HierarchicalForecast benchmark collection introduced by \citet{olivares2022hierarchicalforecast}. That track uses long-horizon autoregressive gated recurrent unit (GRU) models following \citet{cho2014gru}, exact OLS and bottom-up reconciliation, and a validation-tuned blended top-down family.

\paragraph{Metrics.}
We report one-step test MSE, rollout MSE averaged over the horizon, rollout MSE at the final horizon, rollout divergence (Div), stable-rollout length, blow-up rate, and latency where relevant. Unless otherwise noted, numerical entries are reported as mean $\pm$ standard deviation. We denote final-horizon quantities by $\mathrm{MSE}@T_{\mathrm{eval}}$ and $\mathrm{Div}@T_{\mathrm{eval}}$, reserving $H$ and $W$ for spatial grid height and width; concrete tables use numeric horizons such as MSE@20, MSE@60, and MSE@100. This rollout-aware reporting follows the recent shift in neural PDE evaluation toward explicit long-horizon stress tests and autoregressive benchmark protocols \citep{lippe2023pderefiner,koehler2024apebench}. Appendix~\ref{app:hierarchical-support} additionally reports coherence RMS and target-distortion MSE for the hierarchical forecasting support track. Mechanism analysis adds target-distortion diagnostics, divergence-ranked risk--coverage curves, distortion and rollout correlations, controlled mismatch sweeps over cleanup strength, and the appendix-only geometry-aware and projection-frame analyses in Appendix~\ref{app:approx-extensions}.

\paragraph{Recipe.}
Appendix~\ref{app:reproducibility} gives the full reproducibility recipe, covering optimizer, hardware, launch, artifact management, compute, latency, validation model selection, and random-initialization details, including the initialization counts and subset conventions. It also includes the cost comparison used to rule out hidden capacity or training-cost explanations for the main exact-regime ordering and the tuned cylinder cleanup comparison so that the body can focus on dataset scope, the metric definitions above, and fairness-critical comparisons.

\subsection{Main Evidence Across Regimes}
The clearest evidence comes from the exact periodic regime. Table~\ref{tab:main128} shows that Raw FNO is already strong on NS-128, with test MSE $(5.810 \pm 3.600)\times 10^{-8}$, so projection is not beating a weak baseline. Yet exact cleanup is still decisive, matching Theorem~\ref{thm:exact_on_target}. Post hoc projection reduces rollout MSE@100 from $(9.390 \pm 6.290)\times 10^{-5}$ to $(1.130 \pm 0.165)\times 10^{-6}$, and in-loop projection reduces it again to $(5.370 \pm 0.113)\times 10^{-7}$, with CAP statistically tied to Proj-FNO as expected under the exact orthogonal-projector setting. SoftDiv is a tuned negative control rather than a contender. Despite matched compute, it remains worse than raw FNO and far worse than exact projection.

\begin{table}[t!]
\centering
\caption{Main capacity- and compute-matched NS-128 FNO study. Exact in-loop projection is the dominant intervention on the strongest periodic benchmark. Smaller values are better in every numeric column.}
\label{tab:main128}
\resizebox{\linewidth}{!}{%
\begin{tabular}{lccccc}
\toprule
Variant & Test MSE & Rollout MSE-AUC & Rollout MSE@100 & Rollout Div@100 & Latency (ms)\\
\midrule
Raw-FNO & $(5.810 \pm 3.600)\times 10^{-8}$ & $(4.210 \pm 2.750)\times 10^{-5}$ & $(9.390 \pm 6.290)\times 10^{-5}$ & $(2.900 \pm 1.030)\times 10^{-2}$ & $8.75 \pm 0.25$\\
PostHoc-FNO & $(1.570 \pm 0.305)\times 10^{-9}$ & $(6.200 \pm 0.749)\times 10^{-7}$ & $(1.130 \pm 0.165)\times 10^{-6}$ & $(2.670 \pm 0.134)\times 10^{-5}$ & $11.39 \pm 0.26$\\
Proj-FNO & $(8.270 \pm 0.577)\times 10^{-10}$ & $(3.130 \pm 0.029)\times 10^{-7}$ & $(5.370 \pm 0.113)\times 10^{-7}$ & $(2.100 \pm 0.083)\times 10^{-5}$ & $11.49 \pm 0.13$\\
CAP-FNO & $(8.150 \pm 0.244)\times 10^{-10}$ & $(3.180 \pm 0.126)\times 10^{-7}$ & $(5.410 \pm 0.243)\times 10^{-7}$ & $(2.150 \pm 0.252)\times 10^{-5}$ & $11.42 \pm 0.08$\\
SoftDiv-FNO & $(1.550 \pm 1.110)\times 10^{-7}$ & $(1.000 \pm 0.542)\times 10^{-4}$ & $(2.360 \pm 1.370)\times 10^{-4}$ & $(4.550 \pm 1.100)\times 10^{-2}$ & $8.69 \pm 0.24$\\
\bottomrule
\end{tabular}
}
\end{table}

In the approximate regime, the screened and controlled-mismatch experiments already show that operator family matters more than solving the original discrete constraint harder. Mild screened or geometry-aware cleanup can help, but stronger cleanup often moves the trajectory off the learned rollout manifold. The validation-to-deployment protocol below turns this approximate-regime observation into an operating rule, while Appendix~\ref{app:approx-extensions} reports aligned external-backbone and geometry-aware extensions. Together, these results center our study on baseline quality and operator exactness rather than on proposal-side novelty.

Figure~\ref{fig:regimesummary} compresses this split. In the exact regime, projection moves the frontier sharply toward lower divergence and lower error. By contrast, in the approximate regime, the frontier bends. Mild screened or geometry-aware cleanup can help, but stronger cleanup frequently moves the trajectory off the learned rollout manifold.

\begin{figure}[t!]
\centering
\includegraphics[width=\textwidth]{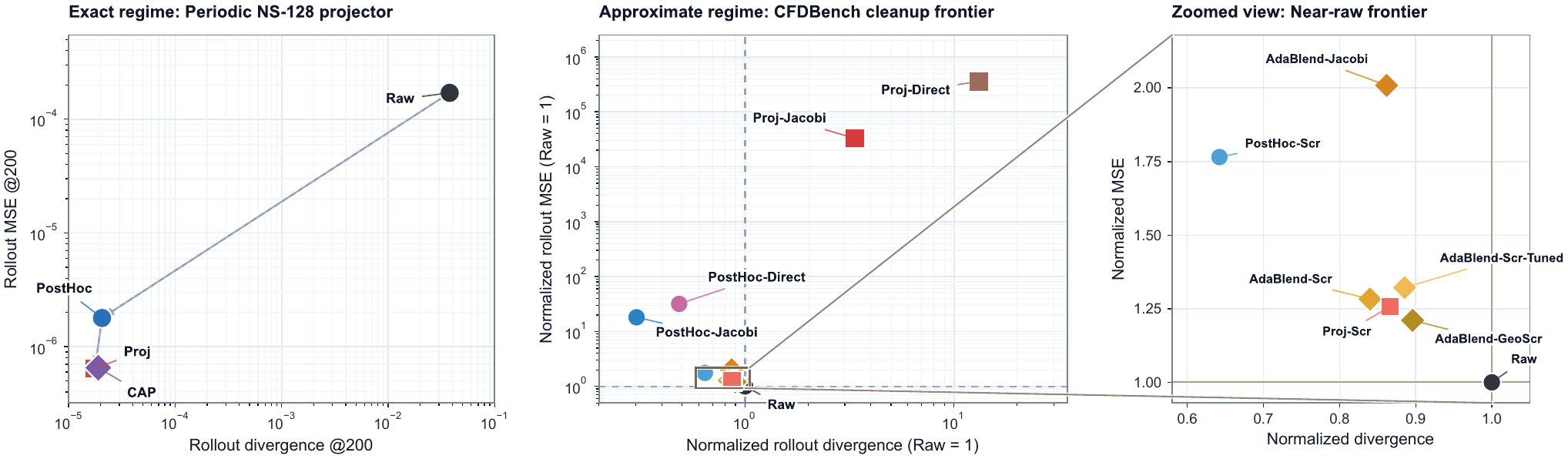}
\caption{Regime split between exact and approximate constraint enforcement. Exact projection behaves like a benign inductive bias on the periodic manifold, while approximate cleanup produces an accuracy--consistency frontier whose shape depends on operator family and dataset. In every subplot, lower divergence and lower error are better.}
\label{fig:regimesummary}
\end{figure}

\subsection{Mechanism Evidence: Distortion, Not Residual, Explains the Reversal}
The core mechanism is visible directly on held-out targets in Table~\ref{tab:projectordiag}. Exact fast Fourier transform (FFT) projection is essentially the identity on periodic ground truth, while approximate cleanup operators are qualitatively different. They often reduce divergence, but, under copied-boundary stencils, the measured full-domain divergence need not decrease for every operator, and they can perturb target states non-negligibly on cavity, tube, and dam. Direct is the cleanest stress test for Proposition~\ref{prop:residual_controls_solver}; it drives the linear-system residual near zero while producing larger target distortion than weaker approximate operators.

\begin{table}[t!]
\centering
\caption{Constraint-operator diagnostics on held-out targets. Exact FFT projection is nearly identity, whereas approximate cleanup trades constraint-solve strength against target distortion; measured full-domain divergence need not decrease for every copied-boundary operator. GT divergence before cleanup is reported once per dataset block because it is shared by all operators; lower values are better in all numeric columns.}
\label{tab:projectordiag}
\small
\resizebox{\linewidth}{!}{%
\begin{tabular}{lcccc}
\toprule
Operator & GT div after & GT distortion MSE & Relative system residual & Distortion relative to energy (\%)\\
\midrule
\multicolumn{5}{@{}l}{\textbf{NS-128 exact periodic, GT div before $=(2.263 \pm 0.450)\times 10^{-6}$}}\\
\cmidrule(lr){1-5}
FFT projection & $(2.787 \pm 0.521)\times 10^{-6}$ & $(8.126 \pm 4.079)\times 10^{-12}$ & not applicable & $0.00$\\
\addlinespace[2pt]
\multicolumn{5}{@{}l}{\textbf{Cavity approximate boundary-aware, GT div before $=(1.840 \pm 1.190)\times 10^{-1}$}}\\
\cmidrule(lr){1-5}
Jacobi-10 & $(6.400 \pm 5.800)\times 10^{-2}$ & $(1.200 \pm 1.500)\times 10^{-2}$ & $(1.954 \pm 0.822)\times 10^{-1}$ & $0.19$\\
MG-2 & $(6.000 \pm 5.600)\times 10^{-2}$ & $(1.500 \pm 1.700)\times 10^{-2}$ & $(1.666 \pm 0.694)\times 10^{-1}$ & $0.25$\\
SOR-10 & $(8.100 \pm 8.400)\times 10^{-2}$ & $(3.000 \pm 3.900)\times 10^{-2}$ & $(2.081 \pm 1.369)\times 10^{-1}$ & $0.66$\\
CG-10 & $(8.200 \pm 8.300)\times 10^{-2}$ & $(3.800 \pm 5.300)\times 10^{-2}$ & $(1.984 \pm 1.011)\times 10^{-1}$ & $0.94$\\
Direct & $(6.200 \pm 6.800)\times 10^{-2}$ & $(4.900 \pm 7.100)\times 10^{-2}$ & $(4.440 \pm 2.765)\times 10^{-6}$ & $1.25$\\
\addlinespace[2pt]
\multicolumn{5}{@{}l}{\textbf{Tube approximate boundary-aware, GT div before $=(5.300 \pm 2.200)\times 10^{-2}$}}\\
\cmidrule(lr){1-5}
Jacobi-20 & $(3.800 \pm 2.100)\times 10^{-2}$ & $(5.928 \pm 2.283)\times 10^{-4}$ & $(2.968 \pm 2.024)\times 10^{-1}$ & $0.25$\\
MG-2 & $(3.700 \pm 1.900)\times 10^{-2}$ & $(5.235 \pm 2.066)\times 10^{-4}$ & $(3.237 \pm 1.912)\times 10^{-1}$ & $0.22$\\
SOR-20 & $(5.300 \pm 1.500)\times 10^{-2}$ & $(2.602 \pm 2.677)\times 10^{-3}$ & $(5.563 \pm 4.560)\times 10^{-1}$ & $0.85$\\
CG-20 & $(5.900 \pm 1.900)\times 10^{-2}$ & $(9.616 \pm 1.087)\times 10^{-3}$ & $(5.357 \pm 4.229)\times 10^{-1}$ & $2.98$\\
Direct & $(5.400 \pm 2.800)\times 10^{-2}$ & $(1.300 \pm 1.400)\times 10^{-2}$ & $(2.945 \pm 2.568)\times 10^{-5}$ & $4.18$\\
\addlinespace[2pt]
\multicolumn{5}{@{}l}{\textbf{Dam approximate boundary-aware, GT div before $=(6.600 \pm 1.000)\times 10^{-2}$}}\\
\cmidrule(lr){1-5}
Jacobi-20 & $(3.200 \pm 0.500)\times 10^{-2}$ & $(2.136 \pm 0.645)\times 10^{-3}$ & $(3.244 \pm 0.360)\times 10^{-1}$ & $0.96$\\
Direct & $(4.400 \pm 0.700)\times 10^{-2}$ & $(1.000 \pm 0.700)\times 10^{-2}$ & $(1.173 \pm 0.335)\times 10^{-5}$ & $4.05$\\
\bottomrule
\end{tabular}
}
\end{table}

Table~\ref{tab:distortionlink} uses the mean-squared-error version of the target-distortion quantity appearing in Theorem~\ref{thm:target_distortion_rollout}. Across the saved approximate-regime sweeps, ground-truth distortion MSE tracks rollout harm strongly, with Pearson $r_{\mathrm{P}}=0.905$ and Spearman $\rho=0.909$, while the relative Poisson residual is negatively correlated with rollout quality. Once the operator is approximate, solving the chosen system more aggressively is not the right default intervention. Appendix~\ref{app:approx-qual} gives the corresponding spatial solver-family mechanism board on a held-out cylinder wake snapshot.

\begin{table}[t!]
\centering
\caption{Across approximate-regime operator sweeps, target distortion explains rollout harm much better than solver residual. All correlations use the same 12 saved sweep points.}
\label{tab:distortionlink}
\begin{tabular}{lcc}
\toprule
Predictor & Pearson $r_{\mathrm{P}}$ & Spearman $\rho$\\
\midrule
GT distortion MSE & $0.905$ & $0.909$\\
GT distortion relative to energy & $0.162$ & $0.329$\\
Poisson residual & $-0.387$ & $-0.427$\\
\bottomrule
\end{tabular}
\end{table}

Appendix~\ref{app:dense-screened-sweep} reports the dense screened-family sweep behind this near-identity reading. We next test the same mechanism at the rollout level through an intervention motivated by Proposition~\ref{prop:controlled_mismatch}, blending raw and cleaned states without retraining as $\tilde{\uu}_{\alpha}=(1-\alpha)\uu_{\mathrm{raw}}+\alpha\uu_{\mathrm{clean}}$.
The result is decisive in Table~\ref{tab:controlledmismatch}. In the matched comparison using the screened operator with $\lambda=16$ (Screened16) and a two-cell post hoc taper (Taper2), Dam@60 is minimized at $\alpha=0$ for both Screened16 and GeoScreened+Taper2, while Cylinder@60 prefers $\alpha=0$ for Screened16 and a barely blended $\alpha=0.1$ for GeoScreened+Taper2. Full cleanup is consistently worse. The approximate-regime sweet spot therefore stays at the raw forecast or very near the identity map rather than at strong cleanup.

\begin{table}[t!]
\centering
\caption{Controlled mismatch study on the hardest long-horizon rows. The best point is raw or a very mild blend, not full cleanup. Smaller values are better in the three MSE columns.}
\label{tab:controlledmismatch}
\resizebox{\linewidth}{!}{%
\begin{tabular}{llcccc}
\toprule
Dataset & Operator & Best $\alpha$ & Best final-step MSE & Mild $\alpha=0.10$ & Full $\alpha=1.0$\\
\midrule
Dam @60 & Screened16 & 0 & $(2.280 \pm 0.122)\times 10^{-4}$ & $(2.320 \pm 0.138)\times 10^{-4}$ & $(6.130 \pm 0.437)\times 10^{-4}$\\
Dam @60 & GeoScreened+Taper2 & 0 & $(2.280 \pm 0.122)\times 10^{-4}$ & $(2.320 \pm 0.136)\times 10^{-4}$ & $(5.580 \pm 0.396)\times 10^{-4}$\\
Cylinder @60 & Screened16 & 0 & $(5.300 \pm 3.400)\times 10^{-3}$ & $(5.400 \pm 3.600)\times 10^{-3}$ & $(1.270 \pm 0.480)\times 10^{-2}$\\
Cylinder @60 & GeoScreened+Taper2 & 0.1 & $(5.300 \pm 3.600)\times 10^{-3}$ & $(5.300 \pm 3.600)\times 10^{-3}$ & $(8.900 \pm 5.200)\times 10^{-3}$\\
\bottomrule
\end{tabular}
}
\end{table}

Appendix~\ref{app:controlled-mismatch-details} reports the corresponding controlled-mismatch curves, and Appendix~\ref{app:approx-qual} reports case-level response-atlas and cylinder alignment-flip diagnostics for the same approximate-regime mechanism.

\paragraph{Validation-to-deployment protocol.}
Table~\ref{tab:deploymentprotocol} makes this protocol concrete on the hardest long-horizon non-periodic rows. Operationally, the protocol starts from one strong raw checkpoint trained under the full benchmark recipe. It then evaluates a small menu of deployment candidates in physical units on validation rollouts, including Raw, adaptive screened and GeoScreened variants, tuned post hoc screened and GeoScreened cleanup, and Direct as a strong-solver anchor. The selected rule is the one with the lowest validation final-horizon rollout MSE for each dataset, and the test report includes that selected rule together with Raw and Direct anchors. The repeated-run robustness extension summarized in Appendix~\ref{app:seed-protocol} leaves the selected rule unchanged.

On the official-split Dam@60 and Cylinder@60 rows, validation keeps the raw forecast rather than any always-clean intervention, exactly matching the controlled-mismatch story that the optimal point is raw or nearly raw. On the geometry holdouts, validation shifts to milder approximate operators instead, with PostHoc-Screened at $\lambda=16$ on Dam and AdaBlend-GeoScreened at $\tau=0.60$ and $q=1.0$ on Cylinder. Across all four validation-selected deployment rows, Direct is never selected; validation chooses either Raw or a milder screened or GeoScreened operator.

Thus, in the approximate regime, the deployment baseline should be selected from among operator families and strengths using rollout validation metrics, rather than defaulting to the strongest physically motivated solve. Appendix~\ref{app:approx-constructive-rule} restates this as a compact validation checklist, while the supporting geometry-aware variants, external-backbone validation, and projection-frame ablation in Appendix~\ref{app:approx-extensions} refine the operating rule without changing its basic selection logic.

\begin{table}[t!]
\centering
\caption{Validation-to-deployment protocol on the four hardest approximate-regime rows. The official-split Dam and Cylinder rows keep the raw forecast, while the geometry holdouts select milder approximate operators rather than Direct cleanup. Smaller values are better in the three MSE columns.}
\label{tab:deploymentprotocol}
\resizebox{\linewidth}{!}{%
\begin{tabular}{lccccc}
\toprule
Setting & Validation-selected rule & Hyperparameters & Selected final-step MSE & Raw final-step MSE & Direct final-step MSE\\
\midrule
Dam (official split) & Raw & none & $(2.400 \pm 0.191)\times 10^{-4}$ & $(2.400 \pm 0.191)\times 10^{-4}$ & $(3.020 \pm 0.240)\times 10^{-2}$\\
Cylinder (official split) & Raw & none & $(4.400 \pm 2.800)\times 10^{-3}$ & $(4.400 \pm 2.800)\times 10^{-3}$ & $(2.380 \pm 0.328)\times 10^{-1}$\\
Dam geometry holdout & PostHoc-Screened & $\lambda=16$ & $(3.540 \pm 0.470)\times 10^{-2}$ & $(3.720 \pm 0.520)\times 10^{-2}$ & $(5.210 \pm 0.400)\times 10^{-2}$\\
Cylinder geometry holdout & AdaBlend-GeoScreened & $\tau=0.60$, $q=1.0$ & $(1.910 \pm 0.480)\times 10^{-2}$ & $(2.210 \pm 0.550)\times 10^{-2}$ & $(9.040 \pm 2.660)\times 10^{-2}$\\
\bottomrule
\end{tabular}
}
\end{table}

\section{Discussion}
In summary, if an exact orthogonal projector matches the benchmark manifold, in-loop projection is the mandatory baseline and usually the dominant intervention. We verified this claim against strong raw baselines, matched compute, long-horizon metrics, repeated-run checks, and external-backbone checks.

The controlled mismatch study makes the approximate-regime claim materially stronger than a simple operator-family comparison. It turns the diagnostic story into a causal one. As soon as a mismatched operator is blended in more aggressively, rollout error rises even while divergence falls, exactly the sign pattern allowed by Proposition~\ref{prop:controlled_mismatch}. Combined with the distortion and residual analysis, the controlled-mismatch evidence isolates the failure mode as operator-induced state distortion rather than insufficient linear-system convergence.

The broader lesson is about operator--data alignment in structured prediction, not only incompressible flow. Appendix~\ref{app:hierarchical-support} provides supporting hierarchical-forecasting evidence. Exact coherent reconciliation transfers as a stable low-risk baseline across Traffic, Labour, TourismLarge, OldTourismLarge, and Wiki2, while the blended top-down family spans the full range from catastrophic to dominant depending on the dataset. This hierarchical evidence is the broader structured-prediction analogue of the CFD finding. Exactness determines whether constraint enforcement acts as a benign inductive bias or as a repeated off-manifold perturbation in autoregressive forecasting. Once the operator is approximate, there is no universal monotone relation between stronger repair and better rollout. Exact constraints can act like a benign inductive bias, but approximate constraints can also act like a recurring source of distribution shift or, more rarely, like a powerful task-specific prior.

This operator-dependent ambiguity is why strong baselines matter so much here. A weak raw model, a short training recipe, or a low-quality long-horizon evaluation can make approximate cleanup look either better or worse than it really is. The evaluation suite used here is designed to reduce that ambiguity while staying within a deliberately limited scope. The strongest evidence remains on 2D incompressible forecasting, and the hierarchical support track broadens the setting without claiming a full map of structured prediction.

\section{Conclusion}
We establish a regime-level evaluation rule for constraint enforcement in autoregressive forecasting. When the constraint operator is exact on the target manifold and orthogonal with respect to the benchmark geometry, as in periodic NS-128, exact projection substantially improves rollout accuracy. Post hoc projection is a mandatory comparator, and in-loop projection is the primary baseline. When exactness is unavailable, as in boundary-aware CFDBench cleanup, stronger enforcement is no longer a safe default. The decisive quantity is not how accurately a Poisson system is solved but how much the chosen operator distorts target trajectories and the learned rollout distribution. Our screened-family sweeps, stronger-solver controls, controlled mismatch experiments, adaptive cleanup tests, and external-backbone checks all support the same deployment rule. Compare operator families in physical units, validate cleanup strength using rollout metrics, and keep raw or near-identity forecasts as serious candidates. The hierarchical forecasting appendix shows the same logic beyond CFD. Exact reconciliation is a stable baseline, while approximate repair can be harmful or beneficial depending on the dataset. Constraint enforcement should therefore be benchmarked by operator--data alignment first and enforcement strength second.

\bibliographystyle{plainnat}
\bibliography{div_refs}

\appendix
\startcontents[apx]
\section*{Appendix Table of Contents}
\printcontents[apx]{l}{1}{\setcounter{tocdepth}{2}}

\section{Theoretical Analysis}
\label{app:theory}
The following results formalize the mechanism separating exact and approximate operators used throughout the experiments. Theorem~\ref{thm:exact_on_target} shows that exact-on-target repair removes additive target distortion. Proposition~\ref{prop:residual_controls_solver} shows that solver residual controls only implementation error around a fixed operator. Theorem~\ref{thm:target_distortion_rollout} shows that trajectorywise target distortion enters rollout bounds directly. Proposition~\ref{prop:controlled_mismatch} shows why raw or near-raw cleanup can be optimal when the cleanup direction is misaligned with the target.

\begin{proposition}[Orthogonal projection identity]
\label{prop:orthogonal_projection_identity}
Let $\mathcal{H}$ be a finite-dimensional real inner product space, let $\mathcal{M}\subset\mathcal{H}$ be a linear subspace, and let $P:\mathcal{H}\to\mathcal{M}$ be the orthogonal projector onto $\mathcal{M}$. Then, for every $x\in\mathcal{H}$ and $y\in\mathcal{M}$, we have
\begin{equation}
\|P x-y\|^2
=
\|x-y\|^2-\|(I-P)x\|^2
\le
\|x-y\|^2.
\label{eq:orthogonal_projection_identity_app}
\end{equation}
Applied to the periodic divergence-free subspace with $P=\proj$, the projection identity is Equation~\ref{eq:orthogonal_proj}.
\end{proposition}

\begin{proof}
Because $P$ is the orthogonal projector, $P x-y\in\mathcal{M}$ and $(I-P)x\in\mathcal{M}^{\perp}$. Since
\begin{equation}
x-y=(P x-y)+(I-P)x,
\end{equation}
the two summands are orthogonal. The Pythagorean theorem therefore gives
\begin{equation}
\|x-y\|^2
=
\|P x-y\|^2+\|(I-P)x\|^2,
\end{equation}
which rearranges to Equation~\ref{eq:orthogonal_projection_identity_app}. The inequality follows because the last term is nonnegative.
\end{proof}

\begin{theorem}[Exact-on-target repair removes additive target distortion]
\label{thm:exact_on_target}
Let $\mathcal{H}$ be a finite-dimensional real inner product space, let $\mathcal{M} \subset \mathcal{H}$ be a target set, and let $T:\mathcal{H}\to\mathcal{H}$ satisfy
\begin{equation}
T(y) = y \text{ for every } y \in \mathcal{M}.
\label{eq:exact_on_target}
\end{equation}
Assume that $T$ is $L_T$-Lipschitz on $\mathcal{H}$. Then, for every $x \in \mathcal{H}$ and every $y \in \mathcal{M}$,
\begin{equation}
\|T(x)-y\|
=
\|T(x)-T(y)\|
\le
L_T \|x-y\|.
\label{eq:exact_on_target_onestep}
\end{equation}
In particular, if $L_T \le 1$, then post hoc repair with $T$ cannot worsen one-step error to any target on $\mathcal{M}$.

Now, let $F:\mathcal{H}\to\mathcal{H}$ be a predictor, let $(y_t)_{t\ge 0} \subset \mathcal{M}$ be a target trajectory, and define the repaired rollout $x_{t+1} = T(F(x_t))$. If $F$ is $L_F$-Lipschitz and $\delta_t := \|F(y_t)-y_{t+1}\|$, then the rollout error satisfies
\begin{equation}
\|x_{t+1}-y_{t+1}\|
\le
L_T L_F \|x_t-y_t\|
+ L_T \delta_t.
\label{eq:exact_rollout_recursion}
\end{equation}
Consequently, for every $t \ge 1$,
\begin{equation}
\|x_t-y_t\|
\le
(L_T L_F)^t \|x_0-y_0\|
+ \sum_{s=0}^{t-1} (L_T L_F)^{t-1-s} L_T \delta_s.
\label{eq:exact_rollout_unrolled}
\end{equation}
If $L_T L_F < 1$ and $\delta_t \le \bar{\delta}$ for all $t$, then
\begin{equation}
\limsup_{t\to\infty}\|x_t-y_t\|
\le
\frac{L_T \bar{\delta}}{1-L_T L_F}.
\label{eq:exact_rollout_steady}
\end{equation}

\end{theorem}

\begin{proof}
Equation~\ref{eq:exact_on_target_onestep} follows immediately from Equation~\ref{eq:exact_on_target} and the Lipschitz property:
\begin{equation}
\|T(x)-y\|
=
\|T(x)-T(y)\|
\le
L_T \|x-y\|.
\end{equation}

For the rollout recursion, because $y_{t+1}\in \mathcal{M}$, Equation~\ref{eq:exact_on_target} gives $T(y_{t+1})=y_{t+1}$. Therefore,
\begin{align}
\|x_{t+1}-y_{t+1}\|
&=
\|T(F(x_t)) - T(y_{t+1})\| \\
&\le
L_T \|F(x_t)-y_{t+1}\| \\
&\le
L_T \|F(x_t)-F(y_t)\|
+ L_T \|F(y_t)-y_{t+1}\| \\
&\le
L_T L_F \|x_t-y_t\|
+ L_T \delta_t,
\end{align}
which is Equation~\ref{eq:exact_rollout_recursion}. Iterating this affine recursion gives Equation~\ref{eq:exact_rollout_unrolled}, and the steady-state bound Equation~\ref{eq:exact_rollout_steady} follows from the geometric series.

\end{proof}

\begin{corollary}[Automatic exactness for linear idempotent repair]
\label{cor:linear_idempotent_exactness}
Let $\mathcal{M}$ be a linear subspace of $\mathcal{H}$. If $T:\mathcal{H}\to\mathcal{H}$ is linear and idempotent with $\mathrm{range}(T)=\mathcal{M}$, then $T(y)=y$ for every $y\in\mathcal{M}$. Thus, Equation~\ref{eq:exact_on_target} holds automatically. The optimal Lipschitz constant is $\|T\|_{\mathrm{op}}$, the operator norm induced by the norm defined by the chosen inner product. Orthogonal projection onto a nonzero target subspace has $\|T\|_{\mathrm{op}}=1$, with the zero-subspace projector as the degenerate norm-zero case, while exact-but-nonorthogonal repair operators can have $\|T\|_{\mathrm{op}}>1$ without introducing any additive target-distortion term.
\end{corollary}

\begin{proof}
For any $y\in\mathcal{M}$, there exists $z$ with $y=Tz$, so $Ty=T^2z=Tz=y$. For linear maps, the optimal Lipschitz constant with respect to the induced norm is $\|T\|_{\mathrm{op}}$.
\end{proof}

\begin{remark}
Theorem~\ref{thm:exact_on_target} and Corollary~\ref{cor:linear_idempotent_exactness} give the abstract version of our exact-regime rule. Exact-but-nonorthogonal operators do not enjoy the Pythagorean identity from the periodic FFT projector, but they still avoid the additive operator-bias term that appears in approximate cleanup bounds. Their risk is conditioning through $\|T\|_{\mathrm{op}}$, not target distortion. This absence of target distortion is why exact bottom-up reconciliation is still a principled baseline even though it is not an orthogonal projector.
\end{remark}

\begin{proposition}[Poisson residual controls only solver error around the chosen operator]
\label{prop:residual_controls_solver}
Let $A_{\lambda}\in\R^{n\times n}$ be symmetric positive definite, let $\mu_{\lambda}:=\lambda_{\min}(A_{\lambda})>0$, and let $D\in\R^{n\times N}$ together with $G\in\R^{N\times n}$ define the cleanup family $Q^{\star}(x) := x - G p^{\star}(x)$, where $A_{\lambda} p^{\star}(x) = -D x$. Let $p(x)$ be any approximate linear-system solve with residual $r(x) := -D x - A_{\lambda} p(x)$, and define the implemented cleanup $Q(x) := x - G p(x)$. Then, for every $x\in\R^N$,
\begin{align}
p(x)-p^{\star}(x) &= -A_{\lambda}^{-1} r(x),
\label{eq:pressure_error_from_residual}\\
Q(x)-Q^{\star}(x) &= G A_{\lambda}^{-1} r(x),
\label{eq:operator_error_from_residual}\\
\|Q(x)-Q^{\star}(x)\|_2
&\le
\frac{\|G\|_2}{\mu_{\lambda}}\|r(x)\|_2,
\label{eq:operator_error_bound_from_residual}\\
\|D Q(x)-D Q^{\star}(x)\|_2
&\le
\frac{\|D G\|_2}{\mu_{\lambda}}\|r(x)\|_2.
\label{eq:divergence_error_bound_from_residual}
\end{align}
Moreover, for every target $y\in\R^N$,
\begin{equation}
\lvert \|Q(x)-y\|_2-\|Q^{\star}(x)-y\|_2\rvert
\le
\frac{\|G\|_2}{\mu_{\lambda}}\|r(x)\|_2,
\label{eq:distance_gap_from_residual}
\end{equation}
and in particular
\begin{equation}
\|Q(y)-y\|_2
\ge
\|Q^{\star}(y)-y\|_2
-
\frac{\|G\|_2}{\mu_{\lambda}}\|r(y)\|_2.
\label{eq:distortion_floor_with_solver_error}
\end{equation}

Consequently, if a sequence of stronger solves $p_k$ satisfies $\|r_k(x)\|_2\to 0$ for every $x$, then $Q_k(x)\to Q^{\star}(x)$ pointwise. A lower linear-system residual therefore shrinks the residual-controlled upper bound on the implementation gap, and residual convergence makes the implementation pointwise faithful to the same operator family $Q^{\star}$; it does not remove any intrinsic target distortion already present in $Q^{\star}$.
\end{proposition}

\begin{proof}
By the definitions of $p^{\star}(x)$ and $r(x)$,
\begin{equation}
A_{\lambda}(p^{\star}(x)-p(x))
=
-D x - A_{\lambda} p(x)
=
r(x),
\end{equation}
which proves Equation~\ref{eq:pressure_error_from_residual}. Substituting into the velocity update gives
\begin{equation}
Q(x)-Q^{\star}(x)
=
-G(p(x)-p^{\star}(x))
=
G A_{\lambda}^{-1} r(x),
\end{equation}
which is Equation~\ref{eq:operator_error_from_residual}. Since $A_{\lambda}$ is symmetric positive definite,
\begin{equation}
\|A_{\lambda}^{-1}\|_2 = \frac{1}{\mu_{\lambda}}.
\end{equation}
Therefore,
\begin{align}
\|Q(x)-Q^{\star}(x)\|_2
&\le
\|G\|_2 \|A_{\lambda}^{-1}\|_2 \|r(x)\|_2
=
\frac{\|G\|_2}{\mu_{\lambda}}\|r(x)\|_2,\\
\|D Q(x)-D Q^{\star}(x)\|_2
&\le
\|D G\|_2 \|A_{\lambda}^{-1}\|_2 \|r(x)\|_2
=
\frac{\|D G\|_2}{\mu_{\lambda}}\|r(x)\|_2,
\end{align}
which proves Equations~\ref{eq:operator_error_bound_from_residual} and \ref{eq:divergence_error_bound_from_residual}. The reverse triangle inequality then yields
\begin{equation}
\lvert \|Q(x)-y\|_2-\|Q^{\star}(x)-y\|_2\rvert
\le
\|Q(x)-Q^{\star}(x)\|_2,
\end{equation}
and Equation~\ref{eq:distance_gap_from_residual} follows from the previous bound. Setting $x=y$ gives Equation~\ref{eq:distortion_floor_with_solver_error}. Finally, if $\|r_k(x)\|_2\to 0$, then Equation~\ref{eq:operator_error_bound_from_residual} implies $\|Q_k(x)-Q^{\star}(x)\|_2\to 0$ for every $x$.
\end{proof}

\begin{remark}
Proposition~\ref{prop:residual_controls_solver} is the formal version of our stronger-solver negative controls. Direct, CG, SOR, MG, and Jacobi can all be compared to the exact application of their own underlying operator family. Lower residual only shrinks the residual-controlled bound in Equation~\ref{eq:operator_error_bound_from_residual}; it does not change the target-distortion floor in Equation~\ref{eq:distortion_floor_with_solver_error}. The residual-distortion separation is exactly why a near-exact solve of a mismatched operator can still lose to a milder screened or geometry-aware family.
\end{remark}

\begin{theorem}[Trajectorywise target distortion enters rollout bounds]
\label{thm:target_distortion_rollout}
Let $\mathcal{H}$ be a finite-dimensional real inner product space, let $\mathcal{M}\subset\mathcal{H}$ be a target set, and let $F:\mathcal{H}\to\mathcal{H}$ be an $L_F$-Lipschitz predictor. For $i\in\{1,2\}$, let $T_i:\mathcal{H}\to\mathcal{H}$ be an $L_i$-Lipschitz cleanup operator. Consider a target trajectory $(y_t)_{t\ge 0}\subset\mathcal{M}$ and the repaired rollouts $x_{t+1}^{(i)} = T_i(F(x_t^{(i)}))$, with $e_t^{(i)} := x_t^{(i)} - y_t$. Define the raw model's one-step target error as $\delta_t := \|F(y_t)-y_{t+1}\|$ and the target-distortion sequence along that trajectory as $\beta_{i,t} := \|T_i(y_t)-y_t\|$, $t\ge 0$. Then, for every $t\ge 0$ and each $i\in\{1,2\}$,
\begin{equation}
\|e_{t+1}^{(i)}\|
\le
L_i L_F \|e_t^{(i)}\|
+ L_i \delta_t
+ \beta_{i,t+1}.
\label{eq:trajectory_distortion_rollout_bound}
\end{equation}
Consequently, if $b_0^{(i)}\ge \|e_0^{(i)}\|$ and
\begin{equation}
b_{t+1}^{(i)} := L_i L_F b_t^{(i)} + L_i \delta_t + \beta_{i,t+1},
\label{eq:bound_recursion}
\end{equation}
then $\|e_t^{(i)}\|\le b_t^{(i)}$ for all $t\ge 0$.
\end{theorem}

\begin{proof}
Fix $i\in\{1,2\}$. Using the rollout definition, we add and subtract $T_i(F(y_t))$ and $T_i(y_{t+1})$, obtaining
\begin{align}
\|e_{t+1}^{(i)}\|
&=
\|T_i(F(x_t^{(i)})) - y_{t+1}\| \\
&\le
\|T_i(F(x_t^{(i)}))-T_i(F(y_t))\|
+ \|T_i(F(y_t))-T_i(y_{t+1})\|
+ \|T_i(y_{t+1})-y_{t+1}\| \\
&\le
L_i \|F(x_t^{(i)})-F(y_t)\|
+ L_i \|F(y_t)-y_{t+1}\|
+ \beta_{i,t+1} \\
&\le
L_i L_F \|e_t^{(i)}\|
+ L_i \delta_t
+ \beta_{i,t+1},
\end{align}
which proves Equation~\ref{eq:trajectory_distortion_rollout_bound}. The recursion in Equation~\ref{eq:bound_recursion} therefore upper-bounds the true error by induction on $t$.
\end{proof}

\begin{corollary}[Smaller target distortion gives a tighter autoregressive bound]
\label{cor:operator_bound_domination}
Under the setup of Theorem~\ref{thm:target_distortion_rollout}, assume
\begin{equation}
\begin{aligned}
b_0^{(1)} &\le b_0^{(2)},\\
L_1 &\le L_2,\\
\beta_{1,t} &\le \beta_{2,t} \text{ for every } t\ge 1.
\end{aligned}
\label{eq:operator_comparison_conditions}
\end{equation}
Then we have
\begin{equation}
b_t^{(1)} \le b_t^{(2)} \text{ for all } t\ge 0.
\label{eq:operator_bound_domination}
\end{equation}
Thus, under matched raw dynamics, any operator family with no larger Lipschitz factor and no larger trajectorywise target distortion along the same target path has a uniformly tighter autoregressive bound at every horizon.
\end{corollary}

\begin{proof}
The claim follows by induction. It holds at $t=0$ by assumption. If it holds at time $t$, then
\begin{align}
b_{t+1}^{(1)}
&=
L_1 L_F b_t^{(1)} + L_1 \delta_t + \beta_{1,t+1} \\
&\le
L_2 L_F b_t^{(1)} + L_2 \delta_t + \beta_{2,t+1} \\
&\le
L_2 L_F b_t^{(2)} + L_2 \delta_t + \beta_{2,t+1} \\
&=
b_{t+1}^{(2)},
\end{align}
which closes the induction.
\end{proof}

\begin{corollary}[Uniform distortion and steady-state bounds]
\label{cor:uniform_distortion_rollout}
Under the setup of Theorem~\ref{thm:target_distortion_rollout}, if the distortion is uniformly bounded on the whole target set,
\begin{equation}
\beta_i := \sup_{y\in\mathcal{M}} \|T_i(y)-y\| < \infty,
\label{eq:uniform_target_distortion}
\end{equation}
then $\beta_{i,t}\le \beta_i$ for every $t$, and therefore
\begin{equation}
\|e_{t+1}^{(i)}\|
\le
L_i L_F \|e_t^{(i)}\|
+ L_i \delta_t
+ \beta_i.
\label{eq:uniform_distortion_rollout_bound}
\end{equation}

Finally, if $L_i L_F<1$, $\delta_t\le \bar{\delta}$ for all $t$, and $\beta_{i,t}\le \bar{\beta}_i$ for all $t$, then
\begin{equation}
\limsup_{t\to\infty}\|e_t^{(i)}\|
\le
\frac{L_i \bar{\delta} + \bar{\beta}_i}{1-L_i L_F}.
\label{eq:uniform_distortion_steady_state}
\end{equation}
Under Equation~\ref{eq:uniform_target_distortion}, one may take $\bar{\beta}_i=\beta_i$. Exact-on-target operators are the special case $\bar{\beta}_i=0$.
\end{corollary}

\begin{proof}
If Equation~\ref{eq:uniform_target_distortion} holds, then $\beta_{i,t+1}\le \beta_i$ for every $t$, and Equation~\ref{eq:uniform_distortion_rollout_bound} follows immediately from Equation~\ref{eq:trajectory_distortion_rollout_bound}.

Finally, if $\delta_t\le \bar{\delta}$, $\beta_{i,t}\le \bar{\beta}_i$, and $L_i L_F<1$, then Equation~\ref{eq:trajectory_distortion_rollout_bound} yields the scalar affine recursion
\begin{equation}
\|e_{t+1}^{(i)}\|
\le
L_i L_F \|e_t^{(i)}\|
+ (L_i \bar{\delta}+\bar{\beta}_i).
\end{equation}
Iterating this recursion and summing the geometric series gives Equation~\ref{eq:uniform_distortion_steady_state}.
\end{proof}

\begin{remark}
Theorem~\ref{thm:target_distortion_rollout} uses a trajectorywise formulation rather than a bare global supremum. For a nonexact linear operator on an unbounded target subspace, the quantity in Equation~\ref{eq:uniform_target_distortion} is typically infinite. The right object for our rollout diagnostics is therefore the distortion actually seen along the benchmark trajectory, with the uniform bound kept only as a stronger corollary when the target set is bounded or the operator is exact on it. Under this formulation, the message stays the same: once the raw predictor $F$ and rollout regime are fixed, the operator family enters the bound through a Lipschitz factor and a target-distortion term, while Proposition~\ref{prop:residual_controls_solver} says that linear-system residual only measures how closely we implement a chosen operator family.
\end{remark}

\begin{proposition}[Controlled-mismatch optimum]
\label{prop:controlled_mismatch}
Let $\mathcal{H}$ be a finite-dimensional real inner product space, let $x \in \mathcal{H}$ be a raw forecast, let $y \in \mathcal{H}$ be the corresponding target, and let $Q:\mathcal{H}\to\mathcal{H}$ be any cleanup operator. Define the cleanup increment $c := Q(x) - x$ and the blended controlled-mismatch family $x_{\alpha} := (1-\alpha)x + \alpha Q(x) = x + \alpha c$, $\alpha \in [0,1]$. Then the one-step squared error $\phi(\alpha) := \|x_{\alpha} - y\|^2$ is the convex quadratic
\begin{equation}
\phi(\alpha)
=
\|x-y\|^2
+ 2 \alpha \langle x-y, c \rangle
+ \alpha^2 \|c\|^2.
\label{eq:blend_quadratic}
\end{equation}
If $c \neq 0$, the minimizing blend strength is therefore
\begin{equation}
\alpha^\star
=
\clip_{[0,1]}
(- \frac{\langle x-y, c \rangle}{\|c\|^2}),
\label{eq:blend_opt}
\end{equation}
while every $\alpha \in [0,1]$ is optimal when $c=0$.

Several immediate consequences follow. If $c \neq 0$ and $\langle x-y, c \rangle \ge 0$, then $\alpha^\star = 0$, so raw forecasting is the unique optimizer, and every nonzero amount of cleanup worsens the one-step MSE.

Full cleanup is no worse than raw forecasting if and only if $2 \langle x-y, c \rangle + \|c\|^2 \le 0$, and it is strictly better if and only if the inequality is strict.
\end{proposition}

\begin{proof}
Set $e := x-y$. Since $x_{\alpha} - y = e + \alpha c$, expanding the square gives
\begin{equation}
\phi(\alpha)
=
\|e+\alpha c\|^2
=
\|e\|^2
+ 2 \alpha \langle e, c \rangle
+ \alpha^2 \|c\|^2,
\end{equation}
which is Equation~\ref{eq:blend_quadratic}. If $c \neq 0$, the blend objective is a strictly convex quadratic, so its unconstrained minimizer is $-\langle e,c\rangle / \|c\|^2$. Clipping to $[0,1]$ yields Equation~\ref{eq:blend_opt}. The raw-optimality claim is then immediate, and the full-cleanup condition follows from
\begin{equation}
\phi(1)-\phi(0) = 2 \langle x-y, c \rangle + \|c\|^2.
\end{equation}
\end{proof}

\begin{corollary}[Controlled mismatch under orthogonal projection]
\label{cor:controlled_mismatch_orthogonal}
In the setup of Proposition~\ref{prop:controlled_mismatch}, suppose $Q=P$ is the orthogonal projector onto a closed subspace $\mathcal{M} \subset \mathcal{H}$ and $y \in \mathcal{M}$. Then,
\begin{equation}
\phi(\alpha)
=
\|P x - y\|^2
+ (1-\alpha)^2 \|(I-P)x\|^2,
\label{eq:exact_blend_curve}
\end{equation}
so $\alpha=1$ is a minimizer, uniquely unless $x \in \mathcal{M}$; when $x\in\mathcal{M}$, every $\alpha\in[0,1]$ is optimal.
\end{corollary}

\begin{proof}
Decompose
\begin{equation}
x-y = (P x-y) + (I-P)x.
\end{equation}
Because $y \in \mathcal{M}$ and $P$ is the orthogonal projector onto $\mathcal{M}$, the first term lies in $\mathcal{M}$ and the second lies in $\mathcal{M}^{\perp}$; hence, they are orthogonal. Moreover, $c = P x - x = -(I-P)x$, so
\begin{equation}
x_{\alpha} - y = (P x-y) + (1-\alpha)(I-P)x.
\end{equation}
Taking norms and using orthogonality proves Equation~\ref{eq:exact_blend_curve}. The right-hand side is minimized at $\alpha=1$, and the minimizer is unique unless $(I-P)x=0$, in which case every $\alpha\in[0,1]$ is optimal.
\end{proof}

\begin{remark}
Proposition~\ref{prop:controlled_mismatch} turns the controlled-mismatch experiment into a precise diagnostic. Because the true residual is $y-x$, any improvement from some nonzero cleanup strength requires $\langle y-x,c\rangle > 0$, equivalently $\langle x-y,c\rangle < 0$. Corollary~\ref{cor:controlled_mismatch_orthogonal} gives the exact orthogonal regime: the cleanup increment is $c=Px-x=-(I-P)x$, so full cleanup cancels the orthogonal-to-subspace component $(I-P)x$. In the approximate regime, this sign is not fixed, which is exactly why the best empirical operating point can move to $\alpha=0$ or a very mild blend.
\end{remark}

\section{Exact Periodic Projection Code}
\label{app:exact-projection-code}
Listing~\ref{lst:periodic-hodge-projection} gives the compact PyTorch implementation of the exact periodic projection used in the NS-128 track.

\begin{lstlisting}[style=pythoncode,caption={Exact periodic Hodge projection for velocity tensors with shape \texttt{(..., 2, H, W)}.},label={lst:periodic-hodge-projection}]
import math
import torch


def spectral_grids(
    height: int,
    width: int,
    device: torch.device,
    domain_size: float = 2 * math.pi,
):
    dx = domain_size / width
    dy = domain_size / height
    kx = 2 * math.pi * torch.fft.rfftfreq(width, d=dx)
    ky = 2 * math.pi * torch.fft.fftfreq(height, d=dy)
    ky_grid, kx_grid = torch.meshgrid(ky, kx, indexing="ij")
    return kx_grid.to(device), ky_grid.to(device)


def curl_from_streamfunction(
    stream: torch.Tensor,
    domain_size: float = 2 * math.pi,
) -> torch.Tensor:
    flat_stream = stream.reshape(-1, *stream.shape[-2:])
    height, width = flat_stream.shape[-2:]
    kx, ky = spectral_grids(height, width, flat_stream.device, domain_size)
    stream_hat = torch.fft.rfft2(flat_stream.float())
    u_hat = 1j * ky * stream_hat
    v_hat = -1j * kx * stream_hat
    u = torch.fft.irfft2(u_hat, s=(height, width))
    v = torch.fft.irfft2(v_hat, s=(height, width))
    velocity = torch.stack([u, v], dim=1)
    return velocity.reshape(*stream.shape[:-2], 2, height, width).to(stream.dtype)


def periodic_hodge_projection(
    velocity: torch.Tensor,
    domain_size: float = 2 * math.pi,
) -> torch.Tensor:
    """Project periodic 2D velocity fields onto the divergence-free subspace."""
    if velocity.shape[-3] != 2:
        raise ValueError("velocity must have shape (..., 2, H, W)")

    original_dtype = velocity.dtype
    leading = velocity.shape[:-3]
    flat = velocity.reshape(-1, *velocity.shape[-3:])
    height, width = flat.shape[-2:]
    kx, ky = spectral_grids(height, width, flat.device, domain_size)

    u = flat[:, 0].float()
    v = flat[:, 1].float()
    u_hat = torch.fft.rfft2(u)
    v_hat = torch.fft.rfft2(v)
    k_sq = kx.square() + ky.square()
    k_sq = k_sq.clone()
    k_sq[0, 0] = 1.0

    vort_hat = 1j * kx * v_hat - 1j * ky * u_hat
    stream_hat = vort_hat / k_sq
    stream_hat[..., 0, 0] = 0.0
    # Real-valued spectral derivatives use the standard zero-Nyquist convention.
    if height % 2 == 0:
        stream_hat[..., height // 2, :] = 0.0
    if width % 2 == 0:
        stream_hat[..., :, width // 2] = 0.0
    stream = torch.fft.irfft2(stream_hat, s=(height, width))
    projected = curl_from_streamfunction(stream, domain_size)
    mean_velocity = flat.mean(dim=(-2, -1), keepdim=True)
    return (projected + mean_velocity).reshape(*leading, 2, height, width).to(original_dtype)
\end{lstlisting}

\section{Forecasting Variant Details}
\label{app:forecasting-variant-details}
On the periodic NS-128 main track, the raw baseline (Raw) uses no projection. The post hoc projection baseline (PostHoc) trains raw and projects each prediction only at evaluation time before feeding it back. The soft-divergence training variant (SoftDiv) optimizes
\begin{equation}
\mathcal{L}_{\mathrm{soft}}
= \mathcal{L}_{\mathrm{vel}}
+ \lambda_{\mathrm{div}}\|\divergence \hat{\uu}\|_2^2
+ \mathcal{L}_{\omega}
+ \mathcal{L}_{\mathrm{spec}}.
\end{equation}
The projected-rollout variant (Proj) projects every predicted state during rollout, and constraint-aware projection (CAP) projects both the autoregressive state and the training target.

On the non-periodic CFDBench tracks, we use adaptive blended cleanup (AdaBlend) to denote the gated variant and include Raw and the Dirichlet-Poisson cleanup baselines AdaBlend-Jacobi, PostHoc-Jacobi, PostHoc-MG, PostHoc-SOR, Proj-Jacobi, Proj-SOR, PostHoc-CG, and Proj-CG. We also add the screened family AdaBlend-Screened, PostHoc-Screened, and Proj-Screened. The projected models are trained with the same operator family and iteration budget used at evaluation time, so the comparison is training-matched rather than an evaluator-only sweep. We use matched budgets of $k=10$ for cavity and $k=20$ for tube, choose the MG V-cycle count by validation search under the same rollout protocol, and use $\lambda=8$ for the main screened-family training runs across the non-periodic experiments. Appendix~\ref{app:approx-extensions} reports additional evaluation-only replicated-boundary, tapered, and GeoScreened stress tests together with the corresponding GeoScreened adaptive variants on the hardest rows.

\section{Approximate-Regime Extensions}
\label{app:approx-extensions}

\subsection{Boundary-Preserving Approximate Cleanup Details}
\label{app:boundary-cleanup-details}
For non-periodic cavity, tube, dam, and cylinder flow, exact spectral projection is unavailable. We therefore use training-matched boundary-preserving approximate cleanup operators. The resulting Poisson solves are approximate, boundary-constrained members of the classical Chorin--Temam projection family rather than exact Leray projectors \citep{brown2001accurate,guermond2006overview,kim1985fractionalstep,bell1989secondorderprojection,perot1993fractionalstep}. Let $\tilde{\uu}$ be the raw predicted velocity. Throughout this discrete cleanup description, we use unit grid spacing; with physical spacings $\Delta x,\Delta y$, the corresponding difference quotients include the usual $1/\Delta x$ and $1/\Delta y$ factors. The interior discrete divergence $D_h\equiv \divergence_h$ is
\begin{equation}
(\divergence_h \tilde{\uu})_{i,j}
= \tilde{u}_{i,j} - \tilde{u}_{i,j-1}
+ \tilde{v}_{i,j} - \tilde{v}_{i-1,j}.
\end{equation}
We solve the positive Dirichlet pressure system on interior cells,
\begin{equation}
\begin{aligned}
A p &= - \divergence_h \tilde{\uu},\\
A &:= -\Delta_h = 4I - S_x^+ - S_x^- - S_y^+ - S_y^-,
\end{aligned}
\end{equation}
where the pressure has zero-Dirichlet boundary values and $S_x^{\pm}$ and $S_y^{\pm}$ are interior shift operators under this boundary convention. The interior velocities are then updated by forward differences, with boundary values copied from the raw prediction:
\begin{align}
\hat{u}_{i,j} &= \tilde{u}_{i,j} - (p_{i,j+1}-p_{i,j}),\\
\hat{v}_{i,j} &= \tilde{v}_{i,j} - (p_{i+1,j}-p_{i,j}),
\end{align}
Equivalently, with $G_{\mathrm{int}}$ denoting this interior-only forward-gradient update, the velocity map is $\hat{\uu}=\tilde{\uu}-G_{\mathrm{int}}p$. Because the boundary velocities are copied, $D_h G_{\mathrm{int}}\neq -A$ on divergence stencils that touch copied boundary velocity values; even an exact solve of $A p=-D_h\tilde{\uu}$ need not make the full interior divergence exactly zero. We denote the resulting operator by $\aproxproj$. Despite the $\widetilde{\Pi}$ notation, these boundary-aware maps are cleanup operators rather than exact projections in the algebraic sense. Copying raw boundary values generally breaks both exact divergence annihilation and idempotence on the full velocity field.

\paragraph{Boundary-aware Poisson--Hodge family.}
We study two approximate operator families. The first is the boundary-aware Poisson--Hodge family. It includes a fixed-budget Jacobi iteration, a geometric multigrid V-cycle that we abbreviate as MG and equip with bilinear restriction, prolongation, and Jacobi smoothing, a red-black successive over-relaxation iteration that we abbreviate as SOR and pair with a finite-budget damped relaxation factor, and a matrix-free conjugate-gradient solve that we abbreviate as CG for this same positive Dirichlet system,
\begin{equation}
A p = - \divergence_h \tilde{\uu}.
\end{equation}
The fifth member of this family is an exact discrete-Dirichlet solve for this same pressure operator, implemented by diagonalizing $A$ in the separable sine basis. Writing the interior pressure as $P \in \R^{(H-2)\times(W-2)}$ and the orthonormal type-I discrete sine transform (DST-I) matrices as $S_{H-2}\in\R^{(H-2)\times(H-2)}$ and $S_{W-2}\in\R^{(W-2)\times(W-2)}$, we solve
\begin{align}
\widehat{R} &= -S_{H-2} (\divergence_h \tilde{\uu})_{\mathrm{int}} S_{W-2}^\top,\\
\widehat{P}_{ij} &= \frac{\widehat{R}_{ij}}{\lambda_i^{(H)} + \lambda_j^{(W)}},\\
P &= S_{H-2}^\top \widehat{P} S_{W-2},
\end{align}
where $\lambda_i^{(H)}$ and $\lambda_j^{(W)}$ are the one-dimensional Dirichlet eigenvalues for $i=1,\dots,H-2$ and $j=1,\dots,W-2$. We refer to the resulting direct cleanup operator as Direct. Jacobi, MG, SOR, CG, and Direct therefore share the same boundary-aware discrete Poisson operator but differ in how aggressively they solve it, letting us separate solver under-convergence from operator mismatch in the approximate regime.

\paragraph{Screened family.}
The second family is a screened variant related to the Helmholtz--Hodge cleanup perspective of \citet{bhatia2013helmholtzhodge}. Instead of solving the pure Poisson system, we solve
\begin{equation}
(A + \lambda I) p = - \divergence_h \tilde{\uu},
\label{eq:screened_hodge}
\end{equation}
with the same boundary-aware stencil and interior-only update $\hat{\uu} = \tilde{\uu} - G_{\mathrm{int}}p$. Equation~\ref{eq:screened_hodge} shrinks the pressure response as the screened shift $\lambda$ increases, tending to make the cleanup milder and, in our sweeps, trading divergence removal against lower target distortion. In the main training-matched comparisons, we use a fixed screened shift $\lambda=8$ and the same training and evaluation iteration budgets as the Jacobi family, so this comparison isolates operator family rather than compute budget. Separately, we run evaluator-side post hoc sweeps over $\lambda\in\{1,2,4,8,16,32,64\}$ on the raw CFDBench checkpoints to test whether the family-level conclusion depends on a single screened setting, and we also run a validation-selected retraining study that chooses $\lambda$ on held-out rollout mean squared error (MSE) area under the curve (AUC), abbreviated MSE-AUC, before retraining the screened projected baseline on the official CFDBench splits.
The following subsections report the dense screened-shift sweep, controlled-mismatch details, boundary-strip audit, and tapered and geometry-aware screened (GeoScreened) evaluator-side variants used in additional approximate-regime stress tests.

\subsection{Dense Screened-Shift Sweep}
\label{app:dense-screened-sweep}
The dense screened-family sweep already suggested that the useful operators are the ones that stay closest to the identity map on the hardest rows. Table~\ref{tab:screenedtrade} shows the pattern cleanly. Tube prefers strong cleanup, cavity sits on a broad intermediate plateau, and the hardest dam and cylinder rows move back toward the raw frontier as screening weakens toward the identity, with Raw still remaining slightly best overall. The table therefore reports the best screened shift among $\lambda>0$ together with the separate raw reference.

\begin{table}[t!]
\centering
\caption{Dense screened-shift sweep. As operator mismatch grows, the best screened cleanup shifts toward larger $\lambda$ and behavior closer to the identity, while the raw reference can still remain best overall on the hardest rows. Smaller values are better in the metric columns.}
\label{tab:screenedtrade}
\resizebox{\linewidth}{!}{%
\begin{tabular}{lcccc}
\toprule
Dataset & Best screened shift with $\lambda>0$ & Raw rollout MSE@20 & Best screened PostHoc rollout MSE@20 & GT distortion MSE\\
\midrule
CFDBench cavity & 16 & $(2.190 \pm 0.170)\times 10^{-1}$ & $(2.160 \pm 0.160)\times 10^{-1}$ & $1.223\times 10^{-4}$\\
CFDBench tube & 1 & $(3.200 \pm 0.900)\times 10^{-2}$ & $(2.900 \pm 0.091)\times 10^{-2}$ & $1.280\times 10^{-4}$\\
CFDBench dam & 64 & $(2.476 \pm 0.118)\times 10^{-4}$ & $(2.499 \pm 0.108)\times 10^{-4}$ & $7.469\times 10^{-7}$\\
CFDBench cylinder & 64 & $(2.000 \pm 0.863)\times 10^{-3}$ & $(2.000 \pm 0.894)\times 10^{-3}$ & $3.430\times 10^{-6}$\\
\bottomrule
\end{tabular}
}
\end{table}

\subsection{Controlled Mismatch Details}
\label{app:controlled-mismatch-details}

Figure~\ref{fig:controlledmismatch} shows the full controlled-mismatch curves corresponding to Table~\ref{tab:controlledmismatch}. The curves make the trade-off behind the table explicit. On Dam@60, both cleanup families move almost immediately away from the raw optimum as $\alpha$ increases. Final-step rollout MSE rises while the cleaned state becomes more divergence-suppressed. Cylinder@60 shows the same pattern for Screened16, and the GeoScreened+Taper2 exception is deliberately small, with only a shallow near-identity minimum at $\alpha=0.1$ before the curve turns upward. Thus, the controlled blend is not selecting a hidden stronger projection; it is exposing that the approximate cleanup direction is mostly misaligned with the rollout-error direction on these long-horizon non-periodic rows. This controlled-mismatch pattern is the visual counterpart of Proposition~\ref{prop:controlled_mismatch}. Once the cleanup increment has the wrong or only weakly favorable alignment with the target residual, increasing enforcement strength can monotonically improve the measured constraint residual while worsening the forecasting objective.

\begin{figure}[t!]
\centering
\includegraphics[width=\textwidth]{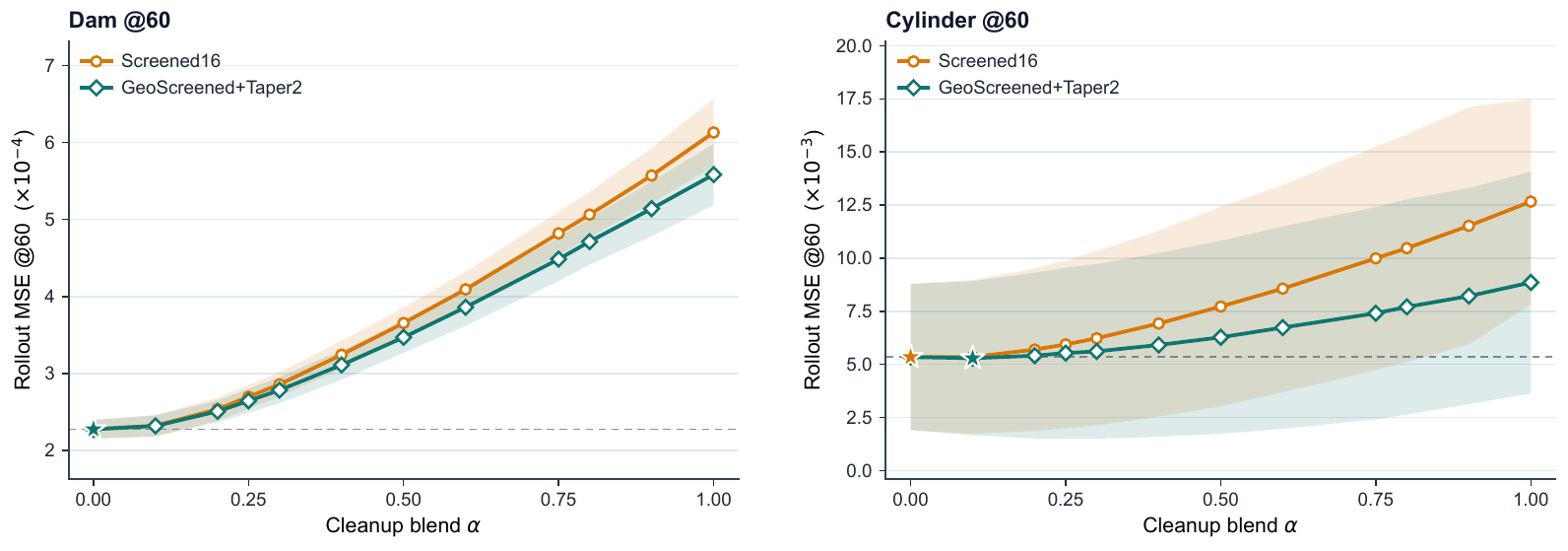}
\caption{Controlled mismatch curves for CFDBench Dam@60 and Cylinder@60, namely the 60-step dam and cylinder long-horizon rows. Increasing cleanup strength lowers divergence but usually worsens rollout error once the operator is mismatched to the data manifold.}
\label{fig:controlledmismatch}
\end{figure}

\subsection{Geometry-Aware Evaluator-Side Variants}
\paragraph{Boundary taper.}
The boundary-strip diagnostics in Table~\ref{tab:boundarydiag} suggest a third, explicitly geometry-aware evaluator-side variant. Let $\mathcal{P}_{\lambda}$ denote the validation-selected screened cleanup, and let $d(i,j)$ be the distance from cell $(i,j)$ to the domain boundary in grid cells. We define a boundary-taper mask
\begin{equation}
M_w(i,j) = \mathrm{clip}((d(i,j)-1)/w, 0, 1), \text{ for } w>0,
\end{equation}
and apply the tapered cleanup
\begin{equation}
\hat{\uu}_{\mathrm{taper}}
=
\tilde{\uu}
+ M_w \odot (\mathcal{P}_{\lambda}(\tilde{\uu}) - \tilde{\uu}).
\end{equation}
In the main tapered evaluation, we use the validation-selected screened shift together with $w=1$, so the first interior ring is left uncorrected, and deeper interior cells retain the screened update. This intentionally weak, evaluation-only taper tests whether the measured boundary-strip mismatch can be turned into a practical accuracy gain on the hardest non-periodic cases.

\begin{table}[t!]
\centering
\caption{Boundary-strip diagnostics on held-out targets under a boundary-focused audit. Stronger cleanup suppresses core divergence more aggressively, but it can also concentrate mismatch near the boundary-adjacent strip, motivating tapered and geometry-aware variants. The distortion column is measured in this strip-audit protocol and is therefore not intended to numerically match the global projector-distortion audit in Table~\ref{tab:projectordiag}. Smaller values are better, except for the strip-to-core ratio, where larger values indicate stronger boundary concentration.}
\label{tab:boundarydiag}
\resizebox{\linewidth}{!}{%
\begin{tabular}{lcccc}
\toprule
Operator regime & Boundary-audit distortion MSE & Boundary-strip div RMS & Core div RMS & Strip-to-core ratio\\
\midrule
Cavity Jacobi-10 & $(7.167 \pm 9.759)\times 10^{-3}$ & $0.122 \pm 0.125$ & $(3.400 \pm 3.400)\times 10^{-2}$ & $(3.600 \pm 0.500)\times 10^{0}$\\
Cavity Direct & $(2.700 \pm 5.500)\times 10^{-2}$ & $0.170 \pm 0.209$ & $(8.139 \pm 9.159)\times 10^{-7}$ & $(2.059 \pm 0.238)\times 10^{5}$\\
Cavity Screened8-10 & $(2.352 \pm 3.041)\times 10^{-4}$ & $0.179 \pm 0.106$ & $(1.190 \pm 0.660)\times 10^{-1}$ & $(1.500 \pm 0.200)\times 10^{0}$\\
Tube Jacobi-20 & $(6.098 \pm 5.333)\times 10^{-4}$ & $0.123 \pm 0.080$ & $(1.600 \pm 1.100)\times 10^{-2}$ & $(1.190 \pm 0.810)\times 10^{1}$\\
Tube Direct & $(1.400 \pm 2.000)\times 10^{-2}$ & $0.210 \pm 0.105$ & $(1.633 \pm 1.407)\times 10^{-6}$ & $(1.848 \pm 0.867)\times 10^{5}$\\
Tube Screened8-20 & $(1.207 \pm 0.801)\times 10^{-5}$ & $0.145 \pm 0.088$ & $(2.600 \pm 1.600)\times 10^{-2}$ & $(7.700 \pm 5.700)\times 10^{0}$\\
Dam Jacobi-20 & $(2.154 \pm 0.644)\times 10^{-3}$ & $0.092 \pm 0.010$ & $(2.300 \pm 0.600)\times 10^{-2}$ & $(4.100 \pm 0.700)\times 10^{0}$\\
Dam Direct & $(1.000 \pm 0.700)\times 10^{-2}$ & $0.178 \pm 0.027$ & $(8.473 \pm 3.723)\times 10^{-7}$ & $(2.282 \pm 0.483)\times 10^{5}$\\
Dam Screened8-20 & $(2.768 \pm 0.790)\times 10^{-5}$ & $0.069 \pm 0.013$ & $(5.900 \pm 0.900)\times 10^{-2}$ & $(1.200 \pm 0.200)\times 10^{0}$\\
Cylinder Jacobi-20 & $(1.600 \pm 2.600)\times 10^{-2}$ & $0.365 \pm 0.250$ & $(7.100 \pm 3.300)\times 10^{-2}$ & $(5.000 \pm 1.600)\times 10^{0}$\\
Cylinder Direct & $(1.000 \pm 1.320)\times 10^{-1}$ & $0.548 \pm 0.309$ & $(2.086 \pm 0.962)\times 10^{-6}$ & $(2.710 \pm 0.860)\times 10^{5}$\\
Cylinder Screened8-20 & $(1.319 \pm 2.676)\times 10^{-4}$ & $0.335 \pm 0.270$ & $(1.340 \pm 0.750)\times 10^{-1}$ & $(2.400 \pm 0.600)\times 10^{0}$\\
\bottomrule
\end{tabular}
}
\end{table}

\paragraph{GeoScreened family.}
The same diagnosis suggests a stronger geometry-aware family that weakens the cleanup within the solve instead of only after the solve. Using the same distance mask, we define a spatially varying screened shift
\begin{equation}
\lambda_w(i,j)
=
\lambda_{\mathrm{bdry}}
+ (\lambda_{\mathrm{core}}-\lambda_{\mathrm{bdry}}) M_w(i,j),
\text{ where } \lambda_{\mathrm{bdry}} \ge \lambda_{\mathrm{core}} \ge 0,
\end{equation}
and solve
\begin{equation}
\begin{aligned}
(A + \Lambda_w) p &= - \divergence_h \tilde{\uu},\\
(\Lambda_w p)_{i,j} &= \lambda_w(i,j) p_{i,j}.
\end{aligned}
\label{eq:geoscreened}
\end{equation}
We refer to the resulting cleanup as the GeoScreened family with boundary shift $\lambda_{\mathrm{bdry}}$, core shift $\lambda_{\mathrm{core}}$, and taper width $w$. When $\lambda_{\mathrm{bdry}}=\lambda_{\mathrm{core}}$, Equation~\ref{eq:geoscreened} reduces to the global screened family. We select these three parameters on validation rollouts and then optionally compose the selected GeoScreened operator with the same two-cell post hoc taper.

\subsection{Qualitative Mechanism Diagnostics}
\label{app:approx-qual}

Figure~\ref{fig:qual-solver-family-mechanism} shows the same residual and distortion failure mode on an actual held-out cylinder wake snapshot. The milder screened family leaves more divergence but preserves the recirculation structure and stays near the raw error level, whereas Jacobi and Direct reduce divergence more aggressively while overcorrecting the wake and spreading much larger velocity error.

\begin{figure}[t!]
\centering
\includegraphics[width=\textwidth]{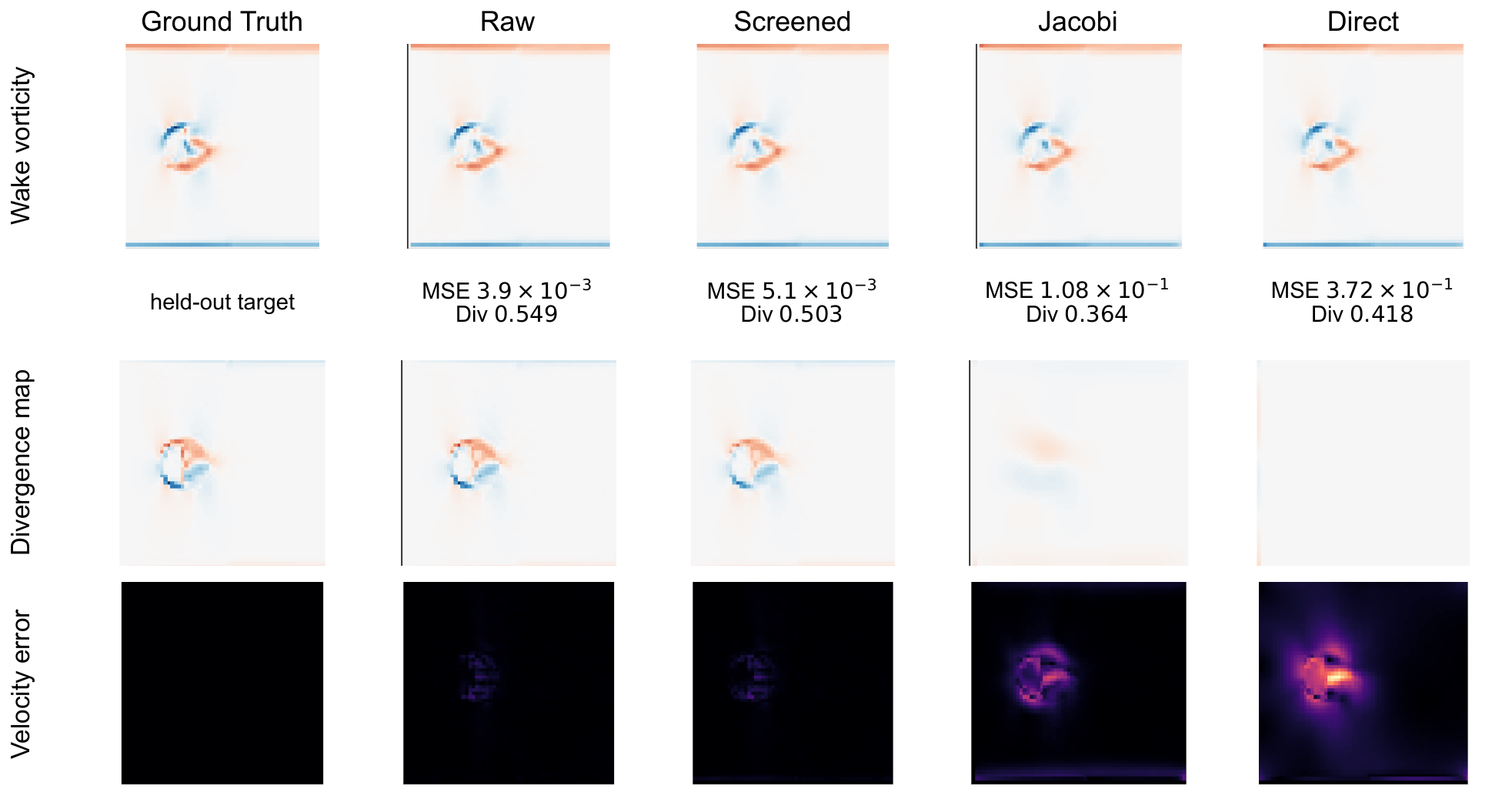}
\caption{
Spatial solver-family mechanism board for a held-out CFDBench cylinder wake snapshot, organized to make the stronger-solver negative control visible. Columns compare Raw, Screened, Jacobi, and Direct against the ground truth; rows show wake vorticity, divergence, and velocity error. The useful qualitative contrast is not only between Raw and Direct. Screened leaves substantially more divergence than Jacobi or Direct, yet it preserves the recirculation structure and stays near the raw error level. Jacobi and Direct both reduce divergence relative to Raw, but both visibly overcorrect the wake and spread much larger error. The solver-family board is the spatial companion to Proposition~\ref{prop:residual_controls_solver}. On the approximate cylinder row, solving the Dirichlet family more faithfully does not recover the right operating point, whereas the milder screened family stays closer to the benchmark manifold.
}
\label{fig:qual-solver-family-mechanism}
\end{figure}

Figures~\ref{fig:qual-case-response-atlas} and \ref{fig:qual-cylinder-alignment-flip} turn the aggregate pattern into case-level diagnostics. The response atlas shows that the approximate-regime split appears in the direction each real trajectory moves in the error--divergence plane, while the paired cylinder comparison shows that the same problem family can flip from cases where Raw is best to cases where mild cleanup is best when the evaluation manifold changes.

\begin{figure}[t!]
\centering
\makebox[\textwidth][c]{\includegraphics[width=1.06\textwidth]{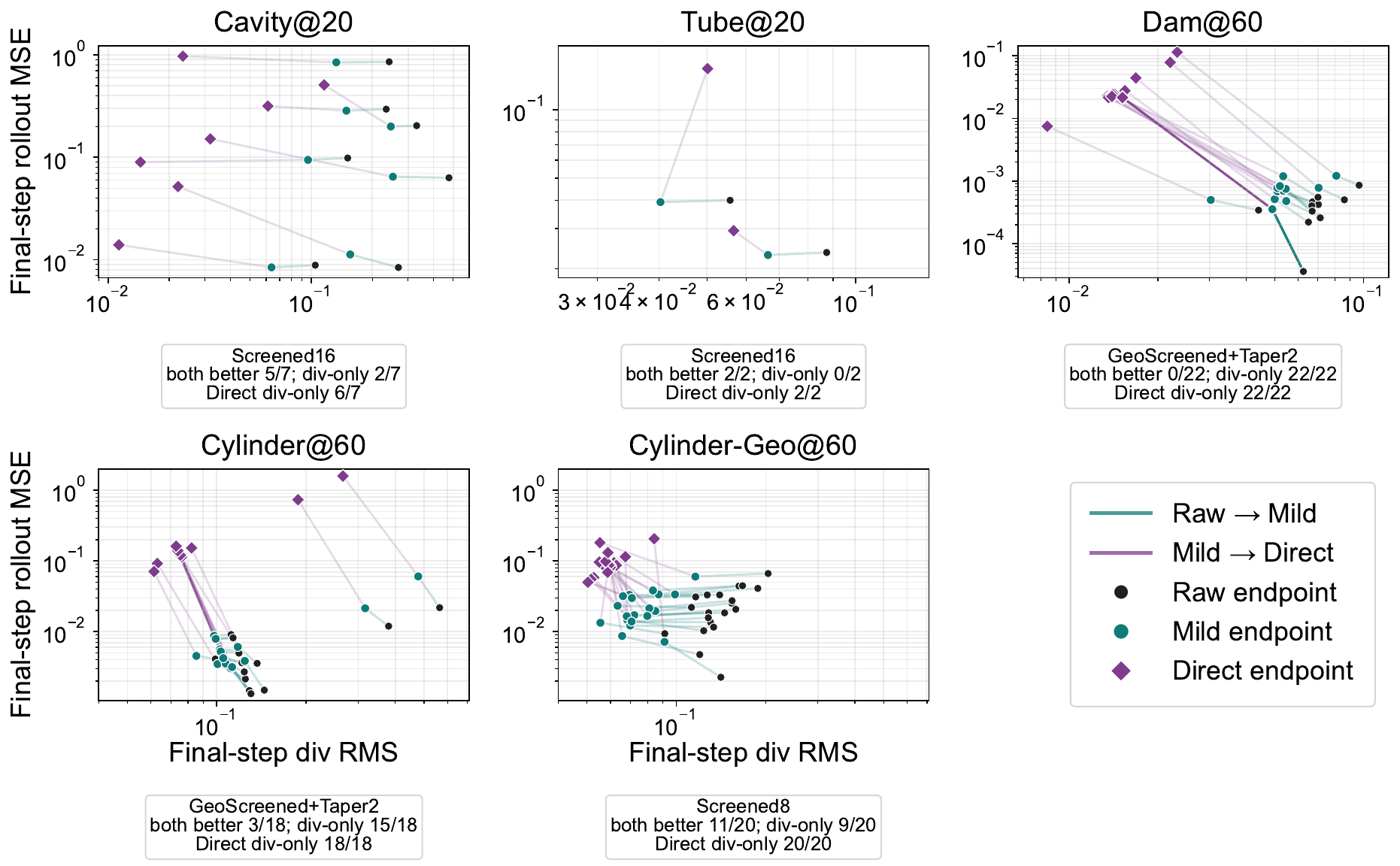}}
\caption{
Case-level response atlas in the final-step error--divergence plane, averaged over seeds 42, 123, and 456. Each polyline is one actual trajectory traced from its Raw endpoint to the dataset-tuned mild operator and then to Direct. The callouts keep the count breakdown compact. Cases improving both metrics have lower final-step MSE and lower divergence relative to Raw, while divergence-only improvements reduce divergence but increase MSE. The aggregate pattern is the point. Tube@20 moves to the corner where both metrics improve under mild cleanup, official-split Dam@60 and Cylinder@60 mostly show divergence-only improvements, and Cylinder-Geo@60 flips back toward improvements in both metrics before Direct turns all cases upward in error. This figure therefore makes our main approximate-regime claim visible on a full set of real cases.
}
\label{fig:qual-case-response-atlas}
\end{figure}

\begin{figure}[t!]
\centering
\includegraphics[width=\textwidth]{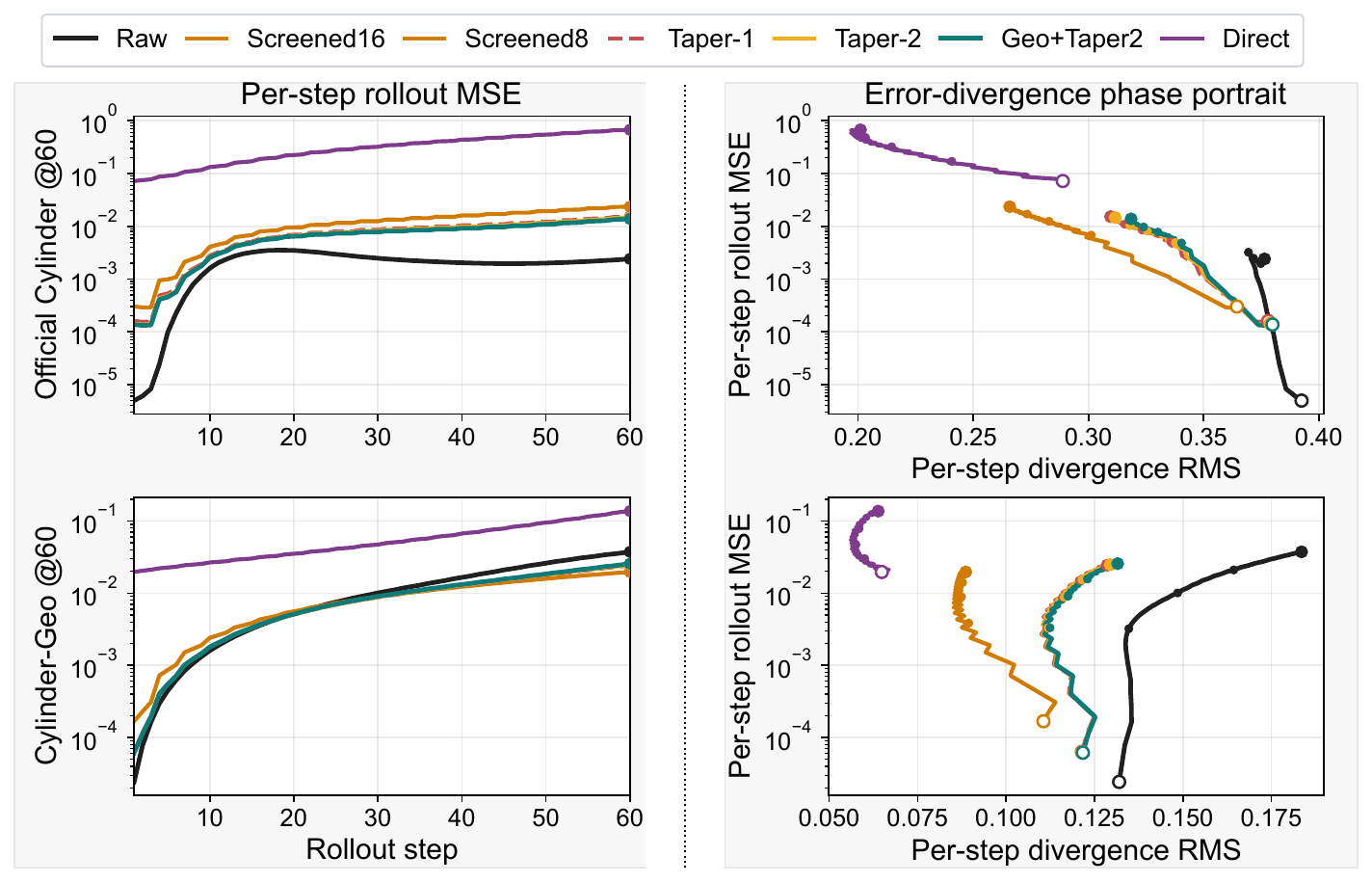}
\caption{
Within-benchmark regime comparison on cylinder flow from the shared seed-42 artifacts. Top row: official-split Cylinder@60 trajectory 1 with the screened curve labeled Screened16. Bottom row: geometry-holdout Cylinder-Geo@60 trajectory 17 with the screened curve labeled Screened8. Left: per-step rollout MSE. Right: the same trajectories traced through the error--divergence plane. In the official-split case, stronger cleanup moves the rollout toward lower divergence but sharply higher error. The Screened16, Taper2, GeoScreened+Taper2, and Direct endpoints all sit above Raw in final-step MSE. On the geometry-holdout case, the ordering flips. Screened8 improves over Raw, Taper2 and GeoScreened+Taper2 remain nearby, and Direct still overcleans. This paired figure turns our main claim into a within-benchmark qualitative comparison. Cleanup strength alone does not determine success; operator alignment with the evaluation manifold does.
}
\label{fig:qual-cylinder-alignment-flip}
\end{figure}

\subsection{Constructive Rule in the Approximate Regime}
\label{app:approx-constructive-rule}
\paragraph{Validation protocol.}
In this paper, validating approximate cleanup strength means selecting both a cleanup family and a cleanup strength based on held-out autoregressive rollout accuracy, not choosing the operator with the smallest constraint residual. Each candidate operator should first be audited on held-out targets in the original physical units, measuring target-distortion MSE $\|\widetilde{\Pi}(y)-y\|_2^2$ together with divergence or coherence reduction. This audit is necessary because a low linear-system residual can still correspond to a mismatched operator that moves benchmark-valid targets.

Raw should remain an explicit deployment candidate, and the comparison should include a strong-solver anchor such as Direct. Raw tests whether cleanup is needed at all, while Direct tests whether solving the nominal constraint harder is actually beneficial. When exact-on-target operators are available, they should be benchmarked before approximate alternatives. When only approximate cleanup is available, the validation menu should span a small family of strengths, including mild or screened operators, blended operators, and adaptive gates.

Cleanup hyperparameters should then be tuned only on validation rollouts under the same horizon and metric used for deployment, such as final-horizon rollout MSE or rollout MSE-AUC, rather than by divergence, Poisson residual, or one-step error alone. The selected test rule should be reported together with Raw and strong-cleanup anchors, including both rollout error and constraint violation, so that the result shows whether cleanup improves accuracy, only improves feasibility, or trades one for the other. A raw or near-identity selection is therefore a valid approximate-regime outcome. It indicates that the available cleanup direction is not aligned well enough with the benchmark manifold at that horizon.

\paragraph{External U-Net validation.}
The same evidence also yields a constructive rule in the approximate regime. Cleanup should not be abandoned outright; family and strength should be chosen explicitly. The screened and GeoScreened families are the best post hoc operating points in our study because they weaken cleanup in a controlled way. This operating-point advantage is visible in the external U-Net validation results in Table~\ref{tab:unetexternal}. On cavity and tube, mild screened cleanup stays on or near the raw frontier, while Jacobi and Direct overclean. Table~\ref{tab:unetexternal} shows that, on Dam@60 and Cylinder@60, GeoScreened+Taper2 becomes the strongest cleanup-only point, reaching $(6.570 \pm 0.850)\times 10^{-4}$ and $(3.800 \pm 0.270)\times 10^{-3}$, respectively, far below Jacobi and Direct.

\begin{table}[t!]
\centering
\caption{Frozen U-Net external validation on the official CFDBench suite. Each cell reports final-horizon rollout MSE on the first line and final-horizon rollout Div on the second line. The upper block reports raw, mild screened cleanup, and GeoScreened cleanup, while the lower block reports stronger-cleanup anchors. All non-raw variants are post hoc cleanup operators.}
\label{tab:unetexternal}
{\small
\resizebox{\linewidth}{!}{%
\begin{tabular}{lccccc}
\toprule
\multicolumn{6}{@{}l}{\textbf{Raw and mild cleanup variants}}\\
\cmidrule(lr){1-6}
Dataset & Raw & Screened16 & Screened16+Taper2 & GeoScreened & GeoScreened+Taper2\\
\midrule
\multicolumn{6}{@{}l}{\textbf{Official suite}}\\
Cavity & \mseDivCell{(2.910 \pm 0.370)\times 10^{-1}}{(2.010 \pm 0.076)\times 10^{-1}} & \mseDivCell{(2.910 \pm 0.370)\times 10^{-1}}{(1.240 \pm 0.032)\times 10^{-1}} & \mseDivCell{(2.910 \pm 0.370)\times 10^{-1}}{(1.400 \pm 0.039)\times 10^{-1}} & \mseDivCell{(2.910 \pm 0.370)\times 10^{-1}}{(1.270 \pm 0.032)\times 10^{-1}} & \mseDivCell{(2.910 \pm 0.370)\times 10^{-1}}{(1.410 \pm 0.040)\times 10^{-1}}\\
Tube & \mseDivCell{(2.100 \pm 0.360)\times 10^{-2}}{(6.200 \pm 0.220)\times 10^{-2}} & \mseDivCell{(2.000 \pm 0.370)\times 10^{-2}}{(5.100 \pm 0.230)\times 10^{-2}} & \mseDivCell{(2.000 \pm 0.360)\times 10^{-2}}{(5.900 \pm 0.200)\times 10^{-2}} & \mseDivCell{(2.000 \pm 0.360)\times 10^{-2}}{(5.300 \pm 0.220)\times 10^{-2}} & \mseDivCell{(2.000 \pm 0.360)\times 10^{-2}}{(6.000 \pm 0.190)\times 10^{-2}}\\
Dam & \mseDivCell{(3.330 \pm 0.200)\times 10^{-4}}{(6.700 \pm 0.042)\times 10^{-2}} & \mseDivCell{(4.280 \pm 0.270)\times 10^{-4}}{(5.200 \pm 0.067)\times 10^{-2}} & \mseDivCell{(4.150 \pm 0.280)\times 10^{-4}}{(5.400 \pm 0.074)\times 10^{-2}} & \mseDivCell{(4.190 \pm 0.280)\times 10^{-4}}{(5.300 \pm 0.070)\times 10^{-2}} & \mseDivCell{(4.150 \pm 0.280)\times 10^{-4}}{(5.500 \pm 0.074)\times 10^{-2}}\\
Cylinder & \mseDivCell{(2.260 \pm 0.180)\times 10^{-3}}{(1.720 \pm 0.009)\times 10^{-1}} & \mseDivCell{(3.040 \pm 0.370)\times 10^{-3}}{(1.320 \pm 0.002)\times 10^{-1}} & \mseDivCell{(2.340 \pm 0.390)\times 10^{-3}}{(1.500 \pm 0.006)\times 10^{-1}} & \mseDivCell{(2.410 \pm 0.400)\times 10^{-3}}{(1.440 \pm 0.005)\times 10^{-1}} & \mseDivCell{(2.310 \pm 0.400)\times 10^{-3}}{(1.520 \pm 0.006)\times 10^{-1}}\\
\addlinespace[2pt]
\midrule
\multicolumn{6}{@{}l}{\textbf{Long-horizon stress rows}}\\
Dam@60 & \mseDivCell{(2.590 \pm 0.320)\times 10^{-4}}{(6.600 \pm 0.075)\times 10^{-2}} & \mseDivCell{(7.180 \pm 0.970)\times 10^{-4}}{(4.200 \pm 0.076)\times 10^{-2}} & \mseDivCell{(6.660 \pm 0.880)\times 10^{-4}}{(4.500 \pm 0.090)\times 10^{-2}} & \mseDivCell{(6.840 \pm 0.890)\times 10^{-4}}{(4.400 \pm 0.081)\times 10^{-2}} & \mseDivCell{(6.570 \pm 0.850)\times 10^{-4}}{(4.600 \pm 0.086)\times 10^{-2}}\\
Cylinder@60 & \mseDivCell{(2.620 \pm 0.081)\times 10^{-3}}{(1.720 \pm 0.014)\times 10^{-1}} & \mseDivCell{(6.140 \pm 0.090)\times 10^{-3}}{(1.170 \pm 0.006)\times 10^{-1}} & \mseDivCell{(4.060 \pm 0.220)\times 10^{-3}}{(1.400 \pm 0.008)\times 10^{-1}} & \mseDivCell{(4.200 \pm 0.250)\times 10^{-3}}{(1.310 \pm 0.003)\times 10^{-1}} & \mseDivCell{(3.800 \pm 0.270)\times 10^{-3}}{(1.440 \pm 0.009)\times 10^{-1}}\\
\bottomrule
\end{tabular}
}

\resizebox{0.78\linewidth}{!}{%
\begin{tabular}{lccc}
\toprule
\multicolumn{4}{@{}l}{\textbf{Strong-cleanup anchors}}\\
\cmidrule(lr){1-4}
Dataset & Raw & Jacobi & Direct\\
\midrule
\multicolumn{4}{@{}l}{\textbf{Official suite}}\\
Cavity & \mseDivCell{(2.910 \pm 0.370)\times 10^{-1}}{(2.010 \pm 0.076)\times 10^{-1}} & \mseDivCell{(3.190 \pm 0.380)\times 10^{-1}}{(3.800 \pm 0.380)\times 10^{-2}} & \mseDivCell{(3.420 \pm 0.400)\times 10^{-1}}{(3.700 \pm 0.710)\times 10^{-2}}\\
Tube & \mseDivCell{(2.100 \pm 0.360)\times 10^{-2}}{(6.200 \pm 0.220)\times 10^{-2}} & \mseDivCell{(4.000 \pm 0.730)\times 10^{-2}}{(2.400 \pm 0.350)\times 10^{-2}} & \mseDivCell{(1.070 \pm 0.190)\times 10^{-1}}{(4.900 \pm 0.600)\times 10^{-2}}\\
Dam & \mseDivCell{(3.330 \pm 0.200)\times 10^{-4}}{(6.700 \pm 0.042)\times 10^{-2}} & \mseDivCell{(7.460 \pm 0.150)\times 10^{-3}}{(2.000 \pm 0.016)\times 10^{-2}} & \mseDivCell{(1.400 \pm 0.025)\times 10^{-2}}{(2.600 \pm 0.018)\times 10^{-2}}\\
Cylinder & \mseDivCell{(2.260 \pm 0.180)\times 10^{-3}}{(1.720 \pm 0.009)\times 10^{-1}} & \mseDivCell{(5.600 \pm 0.210)\times 10^{-2}}{(7.000 \pm 0.059)\times 10^{-2}} & \mseDivCell{(7.500 \pm 0.180)\times 10^{-2}}{(1.130 \pm 0.044)\times 10^{-1}}\\
\midrule
\addlinespace[2pt]
\multicolumn{4}{@{}l}{\textbf{Long-horizon stress rows}}\\
Dam@60 & \mseDivCell{(2.590 \pm 0.320)\times 10^{-4}}{(6.600 \pm 0.075)\times 10^{-2}} & \mseDivCell{(1.200 \pm 0.064)\times 10^{-2}}{(1.100 \pm 0.004)\times 10^{-2}} & \mseDivCell{(2.000 \pm 0.100)\times 10^{-2}}{(1.300 \pm 0.098)\times 10^{-2}}\\
Cylinder@60 & \mseDivCell{(2.620 \pm 0.081)\times 10^{-3}}{(1.720 \pm 0.014)\times 10^{-1}} & \mseDivCell{(1.710 \pm 0.210)\times 10^{-1}}{(5.100 \pm 0.270)\times 10^{-2}} & \mseDivCell{(2.200 \pm 0.018)\times 10^{-1}}{(9.300 \pm 1.100)\times 10^{-2}}\\
\bottomrule
\end{tabular}
}
}
\end{table}

The fairness objection is still important because approximate-regime post hoc gains could reflect a train--test mismatch. The tuned cylinder rows in Table~\ref{tab:cost} show that this concern is only partially valid. On the short-horizon tuned cylinder row, matched projected screened training can indeed beat Raw-FNO, reaching $(1.409 \pm 0.028)\times 10^{-3}$ compared with $(1.642 \pm 0.863)\times 10^{-3}$. However, this tuned-cylinder win is a bounded exception, not a reversal of the main claim. On the harder long-horizon rows for dam and cylinder, the best operating point still stays with raw or very mild post hoc cleanup. Our approximate-regime recommendation is therefore specific. Compare operator families in velocity space, validate cleanup strength explicitly, and expect the winning operator to be much milder than the strongest physically motivated solve.

A final protocol check concerns coordinate-frame choice. Because the non-periodic cleanup operator is implemented in physical velocity units, one could worry that the approximate-regime ordering is really a coordinate-frame artifact. A dedicated normalized and physical frame ablation on post hoc Jacobi and Screened8 shows otherwise. Across cavity, tube, dam, and cylinder, the largest relative change in final-horizon rollout MSE is only $0.14\%$ on tube under Screened8, and most rows move by less than $0.06\%$. We therefore use physical-space cleanup for conceptual consistency, but the main regime split between exact and approximate operators is not driven by this implementation choice.

\section{Hierarchical Forecasting Support}
\label{app:hierarchical-support}
\subsection{Task Setup}
\label{sec:hierarchical-support-task}
As a compact real-data support study outside CFD, we consider autoregressive hierarchical forecasting and test whether the same exact-versus-approximate operator split appears beyond fluid benchmarks. The hierarchical forecasting setup connects our operator-exactness theme to classical, probabilistic, and end-to-end reconciliation settings that are already well-studied outside CFD \citep{taieb2017coherent,panagiotelis2021geometric,rangapuram2021end,athanasopoulos2024review}. Let $x_t \in \R^m$ denote the vector of all hierarchy nodes at time $t$, and let $S \in \{0,1\}^{m \times b}$ be the summing matrix that maps bottom-level series $z_t \in \R^b$ to the full hierarchy:
\begin{equation}
\begin{aligned}
x_t &= S z_t,\\
\mathcal{C} &= \{S z : z \in \R^b\}.
\end{aligned}
\end{equation}
The coherent target manifold is therefore the column space $\mathcal{C}$ of $S$. We train a strong autoregressive gated recurrent unit model following \citet{cho2014gru} to forecast all hierarchy levels jointly on the public Traffic, Labour, TourismLarge, OldTourismLarge, and Wiki2 benchmarks as packaged in the HierarchicalForecast collection introduced by \citet{olivares2022hierarchicalforecast}, then compare three operator families on the predicted states.

The exact OLS reconciliation is the orthogonal projector
\begin{equation}
\Pi_{\mathrm{ols}} = S (S^\top S)^{-1} S^\top.
\end{equation}
We also include the exact-but-nonorthogonal bottom-up repair
\begin{equation}
\Pi_{\mathrm{bu}} = S G_{\mathrm{bu}},
\end{equation}
where $G_{\mathrm{bu}}$ extracts the model's bottom-level forecasts from the full hierarchy vector.

The approximate comparator starts from the full-strength top-down operator built from historical mean proportions $\bar{p} \in \R^b$ with $\mathbf{1}^\top\bar{p}=1$. Let $e_{\mathrm{root}} \in \R^m$ denote the selector for the root node, so that $x^{\mathrm{root}} = e_{\mathrm{root}}^\top x$. Then
\begin{equation}
\widetilde{\Pi}_{\mathrm{td}}(x) = S \bar{p} e_{\mathrm{root}}^\top x.
\end{equation}
We also evaluate a blended variant,
\begin{equation}
\widetilde{\Pi}_{\mathrm{td},\alpha}(x) = x + \alpha(\widetilde{\Pi}_{\mathrm{td}}(x) - x),
\end{equation}
with $\alpha$ selected on the validation split.

\begin{proposition}[Exactness of hierarchical reconciliation operators]
\label{prop:hierarchical_exactness}
Assume that $S$ has full column rank, that $G_{\mathrm{bu}}S=I_b$, and that the root row satisfies $e_{\mathrm{root}}^\top S=\mathbf{1}^\top$. Then
\begin{equation}
\Pi_{\mathrm{ols}}x=x \text{ and } \Pi_{\mathrm{bu}}x=x \text{ for every } x\in\mathcal{C}.
\end{equation}
The full-strength top-down operator $\widetilde{\Pi}_{\mathrm{td}}$ maps every input into $\mathcal{C}$, but, for a coherent point $x=Sz$, it satisfies $\widetilde{\Pi}_{\mathrm{td}}(x)=x$ if and only if $z=(\mathbf{1}^\top z)\bar{p}$. Thus, top-down reconciliation is generally not exact on $\mathcal{C}$ and can distort coherent targets.
\end{proposition}

\begin{proof}
Let $x=Sz\in\mathcal{C}$. For OLS reconciliation,
\begin{equation}
\Pi_{\mathrm{ols}}x
=
S(S^\top S)^{-1}S^\top S z
=
S z
=
x,
\end{equation}
where full column rank makes $S^\top S$ invertible. For bottom-up reconciliation,
\begin{equation}
\Pi_{\mathrm{bu}}x
=
S G_{\mathrm{bu}} S z
=
S z
=
x,
\end{equation}
using $G_{\mathrm{bu}}S=I_b$.

For any input $x$, $\widetilde{\Pi}_{\mathrm{td}}(x)=S(\bar{p}e_{\mathrm{root}}^\top x)$ lies in $\mathrm{range}(S)=\mathcal{C}$, so the full-strength top-down output is coherent. If $x=Sz$ is already coherent, then
\begin{equation}
\widetilde{\Pi}_{\mathrm{td}}(Sz)
=
S\bar{p}e_{\mathrm{root}}^\top S z
=
S\bar{p}\mathbf{1}^\top z.
\end{equation}
Because $S$ has full column rank, the top-down expression equals $Sz$ exactly when $z=(\mathbf{1}^\top z)\bar{p}$. Coherent targets whose bottom-level proportions differ from $\bar{p}$ are therefore moved by $\widetilde{\Pi}_{\mathrm{td}}$.
\end{proof}

These are operator-level exactness statements with respect to the abstract coherent manifold $\mathcal{C}$; the diagnostics reported below are therefore empirical benchmark measurements rather than literal numerical identity checks in the periodic-FFT sense. For general inputs $x$ and $\alpha<1$, the blended top-down operator no longer guarantees coherence, so it is a validation-tuned interpolation toward top-down rather than an exact repair. All hierarchical operators are applied in original value units and only then renormalized for the model, mirroring the physical-space protocol used in the non-periodic CFD study.

\subsection{Results on Real Data}
Hierarchical forecasting provides a real-data testbed outside fluid forecasting. Table~\ref{tab:hierarchicalsupport} evaluates five public hierarchical forecasting benchmarks. A strong autoregressive GRU predicts all hierarchy levels jointly, and we compare operator-level exact coherent reconciliation given by OLS and bottom-up with a validation-tuned blended top-down family. Proposition~\ref{prop:hierarchical_exactness} gives the exactness proof. For coherent targets $x \in \mathcal{C}$, OLS and bottom-up satisfy $\Pi(x)=x$. The small nonzero OLS and Bottom-Up distortion and coherence values reported in Table~\ref{tab:hierarchicalsupport} are coordinate-transform and finite-precision effects of denormalization and renormalization in the evaluation pipeline rather than counterexamples to exactness. The experiment therefore evaluates benchmark-aligned linear repairs, with the exact families serving as the mathematically aligned baseline and the top-down family as the intentionally inexact comparator.

Table~\ref{tab:hierarchicalsupport} shows that the hierarchical evidence does not reduce to a universal ordering. Across Traffic, Labour, TourismLarge, OldTourismLarge, and Wiki2, exact reconciliation is the most stable family. Bottom-Up or Proj-OLS gives the best exact baseline on every dataset and avoids the catastrophic failures seen for approximate top-down repair. The tuned top-down family is highly dataset-dependent. It worsens Traffic and is catastrophic on Labour, where rollout MSE rises to $(3.740 \pm 0.288)\times 10^{3}$, but it becomes decisively best on Wiki2, reducing rollout MSE from Raw-GRU's $(3.170 \pm 0.251)\times 10^{8}$ to $(1.610 \pm 0.112)\times 10^{8}$. TourismLarge and OldTourismLarge lie between these extremes. Bottom-Up remains best, while tuned top-down improves over Raw-GRU but stays worse than the best exact operator. Thus, the support study reinforces the main claim. Exact benchmark-aligned operators are the mandatory low-risk baseline, while approximate operators must be validated dataset by dataset.

This heterogeneity makes the broader point sharper. Exact coherent reconciliation is the safest cross-dataset baseline. It never catastrophically fails across the five benchmarks, and exact Bottom-Up or Proj-OLS is the best exact baseline on all five of them. The blended top-down family has both the highest downside and, on Wiki2, the highest upside. What transfers is therefore not a blanket claim that top-down is always harmful but the more operational lesson that benchmark-aligned exact operators should be the mandatory baseline, while approximate operators must be treated as a separate family whose strength and even sign of benefit are dataset-dependent. In-loop OLS training remains mixed rather than uniformly dominant, matching our main framing. The durable lesson is about baseline quality and operator exactness, not about a universal rule that training with the operator in the loop is always best outside the clean periodic setting.

\begin{table}[t!]
\centering
\caption{Real-data hierarchical forecasting study on five public hierarchies. The OLS and Bottom-Up reconciliation family is the lowest-variance baseline across datasets, while the tuned blended top-down family ranges from harmful to dominant depending on the benchmark. Smaller values are better in every numeric column.}
\label{tab:hierarchicalsupport}
\resizebox{\linewidth}{!}{%
\begin{tabular}{llcccc}
\toprule
Dataset & Variant & Rollout MSE & Last-step MSE & Constraint RMS & Target distortion \\
\midrule
Traffic & Raw-GRU & $(3.663 \pm 0.374)\times 10^{2}$ & $(4.588 \pm 0.754)\times 10^{2}$ & $(3.695 \pm 0.576)\times 10^{0}$ & $0$ \\
 & PostHoc-OLS & $(3.216 \pm 0.373)\times 10^{2}$ & $(4.097 \pm 1.099)\times 10^{2}$ & $(2.000 \pm 0.013)\times 10^{-2}$ & $(4.450 \pm 0.000)\times 10^{-4}$ \\
 & PostHoc-Bottom-Up & $(3.040 \pm 0.311)\times 10^{2}$ & $(3.696 \pm 0.878)\times 10^{2}$ & $(1.800 \pm 0.007)\times 10^{-2}$ & $(9.250 \pm 0.000)\times 10^{-6}$ \\
 & PostHoc-Top-Down & $(4.037 \pm 0.648)\times 10^{2}$ & $(7.321 \pm 1.770)\times 10^{2}$ & $(1.990 \pm 0.161)\times 10^{0}$ & $(2.240 \pm 0.000)\times 10^{-1}$ \\
 & Proj-OLS & $(3.035 \pm 0.481)\times 10^{2}$ & $(4.010 \pm 1.524)\times 10^{2}$ & $(2.000 \pm 0.011)\times 10^{-2}$ & $(4.450 \pm 0.000)\times 10^{-4}$ \\
\midrule
Labour & Raw-GRU & $(9.343 \pm 2.164)\times 10^{2}$ & $(1.830 \pm 0.554)\times 10^{3}$ & $(8.772 \pm 1.488)\times 10^{0}$ & $0$ \\
 & PostHoc-OLS & $(8.157 \pm 2.223)\times 10^{2}$ & $(1.510 \pm 0.558)\times 10^{3}$ & $(3.180 \pm 0.022)\times 10^{-4}$ & $(4.560 \pm 0.000)\times 10^{-8}$ \\
 & PostHoc-Bottom-Up & $(7.657 \pm 1.475)\times 10^{2}$ & $(1.390 \pm 0.307)\times 10^{3}$ & $(2.050 \pm 0.032)\times 10^{-4}$ & $(7.800 \pm 0.000)\times 10^{-9}$ \\
 & PostHoc-Top-Down & $(3.740 \pm 0.288)\times 10^{3}$ & $(8.380 \pm 0.951)\times 10^{3}$ & $(8.360 \pm 1.855)\times 10^{0}$ & $(1.921 \pm 0.000)\times 10^{2}$ \\
 & Proj-OLS & $(6.007 \pm 0.696)\times 10^{2}$ & $(9.824 \pm 1.099)\times 10^{2}$ & $(3.200 \pm 0.030)\times 10^{-4}$ & $(4.560 \pm 0.000)\times 10^{-8}$ \\
\midrule
TourismLarge & Raw-GRU & $(1.890 \pm 0.111)\times 10^{5}$ & $(2.270 \pm 0.291)\times 10^{5}$ & $(1.115 \pm 0.056)\times 10^{2}$ & $0$ \\
 & PostHoc-OLS & $(2.230 \pm 0.086)\times 10^{5}$ & $(2.830 \pm 0.411)\times 10^{5}$ & $(2.360 \pm 0.038)\times 10^{-1}$ & $(7.800 \pm 0.000)\times 10^{-2}$ \\
 & PostHoc-Bottom-Up & $(1.550 \pm 0.068)\times 10^{5}$ & $(1.660 \pm 0.261)\times 10^{5}$ & $(2.610 \pm 0.033)\times 10^{-1}$ & $(6.600 \pm 0.000)\times 10^{-3}$ \\
 & PostHoc-Top-Down & $(1.710 \pm 0.151)\times 10^{5}$ & $(1.210 \pm 0.298)\times 10^{5}$ & $(8.402 \pm 0.388)\times 10^{1}$ & $(1.084 \pm 0.000)\times 10^{2}$ \\
 & Proj-OLS & $(1.620 \pm 0.060)\times 10^{5}$ & $(1.360 \pm 0.077)\times 10^{5}$ & $(2.350 \pm 0.006)\times 10^{-1}$ & $(7.800 \pm 0.000)\times 10^{-2}$ \\
\midrule
OldTourismLarge & Raw-GRU & $(1.800 \pm 0.086)\times 10^{5}$ & $(2.220 \pm 0.368)\times 10^{5}$ & $(1.115 \pm 0.049)\times 10^{2}$ & $0$ \\
 & PostHoc-OLS & $(2.170 \pm 0.115)\times 10^{5}$ & $(2.920 \pm 0.488)\times 10^{5}$ & $(2.370 \pm 0.042)\times 10^{-1}$ & $(7.800 \pm 0.000)\times 10^{-2}$ \\
 & PostHoc-Bottom-Up & $(1.470 \pm 0.090)\times 10^{5}$ & $(1.570 \pm 0.335)\times 10^{5}$ & $(2.640 \pm 0.046)\times 10^{-1}$ & $(6.600 \pm 0.000)\times 10^{-3}$ \\
 & PostHoc-Top-Down & $(1.640 \pm 0.094)\times 10^{5}$ & $(1.080 \pm 0.271)\times 10^{5}$ & $(8.352 \pm 0.320)\times 10^{1}$ & $(1.084 \pm 0.000)\times 10^{2}$ \\
 & Proj-OLS & $(1.590 \pm 0.032)\times 10^{5}$ & $(1.270 \pm 0.100)\times 10^{5}$ & $(2.310 \pm 0.029)\times 10^{-1}$ & $(7.800 \pm 0.000)\times 10^{-2}$ \\
\midrule
Wiki2 & Raw-GRU & $(3.170 \pm 0.251)\times 10^{8}$ & $(4.320 \pm 0.599)\times 10^{8}$ & $(3.950 \pm 0.479)\times 10^{3}$ & $0$ \\
 & PostHoc-OLS & $(3.030 \pm 0.200)\times 10^{8}$ & $(4.090 \pm 0.516)\times 10^{8}$ & $(2.604 \pm 0.220)\times 10^{0}$ & $(1.329 \pm 0.000)\times 10^{1}$ \\
 & PostHoc-Bottom-Up & $(3.130 \pm 0.201)\times 10^{8}$ & $(4.190 \pm 0.488)\times 10^{8}$ & $(2.404 \pm 0.251)\times 10^{0}$ & $(1.601 \pm 0.000)\times 10^{1}$ \\
 & PostHoc-Top-Down & $(1.610 \pm 0.112)\times 10^{8}$ & $(1.670 \pm 0.224)\times 10^{8}$ & $(3.199 \pm 5.985)\times 10^{1}$ & $(9.330 \pm 1.790)\times 10^{7}$ \\
 & Proj-OLS & $(2.770 \pm 0.031)\times 10^{8}$ & $(3.270 \pm 0.039)\times 10^{8}$ & $(2.885 \pm 0.163)\times 10^{0}$ & $(1.329 \pm 0.000)\times 10^{1}$ \\
\bottomrule
\end{tabular}
}
\end{table}

\section{Limitations}
The strongest evidence is still on 2D velocity benchmarks. The hierarchical forecasting experiments broaden our coverage outside CFD, but they remain compact. We do not claim to have mapped the full structured-prediction design space outside CFD. CFDBench also broadens the non-periodic regime substantially without exhausting it, and richer geometry-aware or mesh-aware families may move the frontier further.

\section{Reproducibility}
\label{app:reproducibility}
Table~\ref{tab:cost} makes the cost structure explicit. On the exact NS-128 FNO track, Raw, SoftDiv, PostHoc, Proj, and CAP all use the same $18.9$M-parameter backbone and nearly identical train hours, so the main exact-regime ordering is not a hidden capacity effect. The tuned cylinder cleanups show the complementary fairness point. AdaBlend-Screened-Tuned and PostHoc-Screened-Tuned reuse the raw checkpoint exactly, whereas Proj-Screened-Tuned raises training time from $0.49$ to $0.64$ hours because it retrains with the projected operator in the loop.

\begin{table}[t!]
\centering
\caption{Cost comparison for the exact-regime and tuned cylinder settings. Evaluation-time cleanup baselines reuse the source checkpoint's training cost, while projected retraining pays its extra cost explicitly. Smaller values are better in every numeric column.}
\label{tab:cost}
\resizebox{\linewidth}{!}{%
\begin{tabular}{llcccc}
\toprule
Track & Variant & Params & Train hours & Latency (ms) & Final-step rollout MSE\\
\midrule
NS-128 exact & Raw-FNO & $18.90$M & $0.84 \pm 0.08$ & $8.75 \pm 0.25$ & $(9.389 \pm 6.291)\times 10^{-5}$\\
NS-128 exact & SoftDiv-FNO & $18.90$M & $0.84 \pm 0.07$ & $8.69 \pm 0.24$ & $(2.361 \pm 1.371)\times 10^{-4}$\\
NS-128 exact & PostHoc-FNO & $18.90$M & $0.84 \pm 0.08$ & $11.39 \pm 0.26$ & $(1.133 \pm 0.165)\times 10^{-6}$\\
NS-128 exact & Proj-FNO & $18.90$M & $0.88 \pm 0.08$ & $11.49 \pm 0.13$ & $(5.370 \pm 0.113)\times 10^{-7}$\\
NS-128 exact & CAP-FNO & $18.90$M & $0.88 \pm 0.08$ & $11.42 \pm 0.08$ & $(5.406 \pm 0.243)\times 10^{-7}$\\
\midrule
Cylinder-20 tuned & Raw-FNO & $6.44$M & $0.49 \pm 0.00$ & $5.50 \pm 0.01$ & $(1.642 \pm 0.863)\times 10^{-3}$\\
Cylinder-20 tuned & AdaBlend-Screened-Tuned & $6.44$M & $0.49 \pm 0.00$ & $11.79 \pm 0.05$ & $(2.540 \pm 1.009)\times 10^{-3}$\\
Cylinder-20 tuned & PostHoc-Screened-Tuned & $6.44$M & $0.49 \pm 0.00$ & not reported & $(3.164 \pm 0.872)\times 10^{-3}$\\
Cylinder-20 tuned & Proj-Screened-Tuned & $6.44$M & $0.64 \pm 0.00$ & $12.00 \pm 0.36$ & $(1.409 \pm 0.028)\times 10^{-3}$\\
\bottomrule
\end{tabular}
}
\end{table}

\paragraph{Training and Systems Details.}
All CFD backbones use AdamW, cosine decay, mixed precision, and validation-based checkpoint selection on four NVIDIA A100 GPU accelerators. We use PyTorch for training and evaluation. Appendix~\ref{app:hierarchical-support} uses an autoregressive GRU recipe for the hierarchical forecasting support track but keeps the same long-horizon validation protocol. For the long-trajectory cylinder benchmark, we subsample only the training windows with stride 16 while keeping validation and test rollout evaluation at stride 1, so the stronger cylinder recipe remains compute-matched without weakening the evaluation protocol.

\paragraph{Random Initialization Protocol.}
\label{app:seed-protocol}
The main CFD experiment blocks in the main text, including the exact-regime tables, projector diagnostics, screened-family sweep, and frozen U-Net validation, use three random initializations unless noted otherwise in the underlying analysis ledger. The hierarchical forecasting study uses five random initializations on each of the five public hierarchies. The validation-to-deployment protocol on the four hardest approximate-regime rows also uses five random initializations per setting as a robustness extension. Table~\ref{tab:controlledmismatch} uses the common three-run subset shared across all controlled-mismatch cleanup strengths so that the blends are compared on identical underlying checkpoints. The larger five-run deployment extension keeps the selected validation rule unchanged on all four hardest rows.

\section{List of Notation}
\label{app:notation}
\begin{table}[!htbp]
\caption{Notation summary.}
\label{tab:notation}
\centering
\scriptsize
\resizebox{\linewidth}{!}{%
\begin{tabular}{ll}
\toprule
Symbol & Meaning\\
\midrule
$\R,\mathcal{H},\mathcal{M}$ & real field, abstract state space, target manifold\\
$\uu=(u,v)$ & two-dimensional velocity field\\
$\tilde{\uu},\hat{\uu},\yy$ & raw prediction, cleaned prediction, target velocity\\
$x_t,y_t,e_t$ & rollout state, target state, rollout error\\
$t,s,K,T_{\mathrm{in}},T_{\mathrm{eval}}$ & time index, rollout step, model rollout length, input length, evaluation horizon\\
$H$ & spatial grid height\\
$W$ & spatial grid width\\
$b_\theta,\Delta\tilde{\uu},F$ & neural backbone, residual update, abstract predictor\\
$\alpha,\alpha^\star$ & residual or blend weight, optimal blend\\
$\mathcal{T},T,Q,Q^\star$ & rollout repair operator, abstract cleanup, exact family member\\
$\proj,\aproxproj,\widetilde{\Pi}$ & exact periodic projector, boundary-aware cleanup, generic approximate cleanup\\
$\mathcal{P},\mathcal{P}_\lambda$ & selected post hoc cleanup, screened cleanup\\
$\divergence,\divergence_h,\nabla_h,\Delta_h$ & continuous divergence, discrete divergence, gradient, and Laplacian\\
$\kk=(k_x,k_y),\widehat{\cdot},\bar{\uu}$ & Fourier wavevector, Fourier coefficient, spatial mean\\
$\omega,\psi$ & vorticity, stream function\\
$p,P,A,A_\lambda,\lambda$ & pressure, pressure matrix, Poisson operator, screened shift\\
$S_x^\pm,S_y^\pm,S_{H-2},S_{W-2}$ & grid shifts, DST-I matrices\\
$\lambda_i^{(H)},\lambda_j^{(W)},\mu_\lambda$ & Dirichlet eigenvalues, smallest eigenvalue\\
$D,G,r(x)$ & abstract divergence, gradient, linear-system residual\\
$\mathcal{L}_{\mathrm{soft}},\mathcal{L}_{\mathrm{vel}},\mathcal{L}_\omega,\mathcal{L}_{\mathrm{spec}}$ & soft-divergence, velocity, vorticity, spectrum losses\\
$\lambda_{\mathrm{div}},k$ & divergence penalty weight, cleanup iteration budget\\
$d(\tilde{\uu}),\gamma,\tau,q$ & divergence score, adaptive gate, threshold, gate exponent\\
$M_w,d(i,j),w,\odot$ & taper mask, boundary distance, taper width, pointwise product\\
$\lambda_w,\Lambda_w,\lambda_{\mathrm{bdry}},\lambda_{\mathrm{core}}$ & spatial, diagonal, boundary, and core screened shifts\\
$L,L_T,L_F,L_i$ & Lipschitz constants\\
$\delta_t,\beta_{i,t},\beta_i$ & raw one-step error, trajectory, and uniform target distortion\\
$b_t^{(i)},\bar{\delta},\bar{\beta}$ & error-bound sequence, uniform error and distortion bounds\\
$c,x_\alpha,\phi(\alpha)$ & cleanup increment, blended forecast, blend-error curve\\
$I,\|\cdot\|_2,\langle\cdot,\cdot\rangle,(\cdot)^\top$ & identity, Euclidean norm, inner product, transpose\\
$S,z_t,m,b,\mathcal{C}$ & hierarchy summing matrix, bottom series, node counts, coherent manifold\\
$\Pi_{\mathrm{ols}},\Pi_{\mathrm{bu}},\widetilde{\Pi}_{\mathrm{td}},\widetilde{\Pi}_{\mathrm{td},\alpha}$ & OLS, bottom-up, top-down, blended top-down reconciliation\\
$G_{\mathrm{bu}},\bar{p},e_{\mathrm{root}}$ & bottom extractor, historical proportions, root selector\\
$r_{\mathrm{P}},\rho,n$ & Pearson correlation, Spearman correlation, sample count\\
$\mathrm{MSE}@T_{\mathrm{eval}},\mathrm{Div}@T_{\mathrm{eval}},\mathrm{AUC}$ & final-step error, final-step divergence, rollout area metric\\
\bottomrule
\end{tabular}
}
\end{table}

\end{document}